\documentclass{article}
\usepackage{iclr2027_conference,times}
\usepackage{amsmath,amsfonts,bm}

\def\eqref#1{(\ref{#1})}
\def\1{\bm{1}}

\DeclareMathAlphabet{\mathsfit}{\encodingdefault}{\sfdefault}{m}{sl}
\SetMathAlphabet{\mathsfit}{bold}{\encodingdefault}{\sfdefault}{bx}{n}

\usepackage{amsmath,amssymb,mathtools}
\usepackage{amsthm}
\usepackage{graphicx}
\usepackage{booktabs}
\usepackage{multirow}
\usepackage{array}
\usepackage{float}
\usepackage{microtype}
\usepackage{hyperref}
\usepackage{url}
\usepackage{xcolor}

\newtheorem{theorem}{Theorem}[section]
\newtheorem{lemma}[theorem]{Lemma}

\definecolor{mydarkblue}{rgb}{0,0.08,0.45}
\definecolor{mydarkgreen}{RGB}{0, 139, 69}

\hypersetup{
	colorlinks=true,
	urlcolor=magenta,
	citecolor=mydarkblue,
}

\title{Large Diffusion Models Need More Than EMA}
\author{Chuyu Hu$^{1,3}$, Huanran Chen$^{1}$, Shitong Shao$^{2}$, Yinpeng Dong$^{1}$\\[3pt]
\normalfont $^{1}$Tsinghua University\\
\normalfont $^{2}$The Hong Kong University of Science and Technology\\
\normalfont $^{3}$Shanghai Qi Zhi Institute}

\iclrfinalcopy
\begin{document}
\maketitle
\fancyhead[L]{}
\begingroup
\renewcommand{\thefootnote}{}
\footnotetext{Code: \href{https://github.com/Mosy-hubhub/REPA/tree/REPA-schedulefree}{SiT+REPA pretraining}; \href{https://github.com/Mosy-hubhub/DiffusionOPSD/tree/ReFL-Schedulefree}{ReFL (DiffusionOPSD)}; \href{https://github.com/Mosy-hubhub/dmdr/tree/schedule-dmdr}{DMD (DMDR)}.}
\endgroup

\begin{abstract}
Large-scale diffusion models are commonly trained with a constant learning rate (LR) and exponential moving average (EMA), despite the strong performance of LR decay and Schedule-Free optimization in training large language models. We study whether EMA is uniquely suited to diffusion training, or whether stronger alternatives have been underexplored. Through the lens of variance reduction, we theoretically show that averaging a constant-step trajectory has an inherent optimization limitation, whereas LR decay and Schedule-Free optimization admit vanishing stationarity guarantees under standard nonconvex assumptions. Empirically, LR decay and EMA provide complementary rather than substitutable benefits: decay alone does not outperform EMA with a constant LR, while combining decay with EMA further improves this standard baseline at most tested starting points, although the gain remains sensitive to when decay begins. Schedule-Free avoids this timing requirement and consistently outperforms EMA with a constant LR across training stages, achieving both lower fixed-evaluation objectives and better generation quality. Our analysis further shows that all three strategies substantially reduce parameter variability, yet stronger stabilization is not always better: EMA with a constant LR retains comparatively higher parameter variability, while aggressive decay combined with EMA achieves the lowest checkpoint variability but also weakens continued optimization. Schedule-Free substantially reduces parameter variability while preserving more late-stage progress, providing a better balance between parameter stability and continued optimization. Finally, the benefits of Schedule-Free extend beyond pretraining to ReFL and generator-side DMD, demonstrating its transferability across distinct diffusion optimization settings. Overall, our results suggest that large diffusion models need more than parameter averaging alone: effective optimization requires balancing variance reduction with continued optimization.
\end{abstract}

\section{Introduction}
\label{sec:introduction}

Diffusion architectures have advanced rapidly from the convolutional U-Net backbones used in DDPMs \citep{ronneberger2015unet,ho2020denoising} to Diffusion Transformers such as DiT and SiT \citep{peebles2023scalable,ma2024sit}, alongside improvements in training objectives and representation learning~\citep{karras2022elucidating,yu2025repa}. In contrast, their optimization recipes have changed much less. Many large-scale diffusion models are still trained with a constant learning rate (LR) and evaluated using exponential moving average (EMA) weights, even though LR decay is standard in large-scale language-model pretraining~\citep{hu2024minicpm,wen2025wsd} and Schedule-Free optimization has recently shown strong performance in the same regime~\citep{defazio2024road}.

This discrepancy matters because improving the training recipe can reduce the compute needed to obtain a strong model without changing the architecture, objective, or data. Much recent work on accelerating diffusion training has instead focused on these other components. For example, REPA substantially accelerates Diffusion Transformer training by improving representation learning through alignment with pretrained visual features~\citep{yu2025repa}. By comparison, the optimization recipe itself has received less systematic attention, especially for modern large-scale diffusion-transformer pretraining. As a result, it remains unclear whether the widespread use of constant LR with EMA reflects a genuine optimization advantage or simply a largely inherited training convention.

This motivates our central question: \textbf{does EMA have an optimization advantage specific to large-scale diffusion training, or have stronger alternatives such as LR decay and Schedule-Free optimization simply been underexplored?}

We study these strategies from a unified variance-reduction perspective.
We argue that EMA, LR decay, and Schedule-Free can be understood as
three different approaches to the same underlying problem: persistent variability
in stochastic optimization. Their distinction is not simply whether they stabilize
training, but how they reduce variability and how this interacts with continued
optimization.
EMA reduces variability by averaging a constant-step trajectory while leaving the
underlying optimization dynamics unchanged. LR decay stabilizes the
trajectory more directly by shrinking the updates, but this simultaneously reduces
the magnitude of further optimization steps. Schedule-Free instead combines
continued large-step optimization with averaging over an increasingly long
trajectory, reducing variability without suppressing the underlying optimization
dynamics as aggressively.

Our theoretical analysis places these three strategies within a common framework,
clarifying how each reduces variability and exposing their different trade-offs
between trajectory stability and continued optimization.
In particular, averaging a constant-step trajectory has an inherent optimization
limitation, whereas LR decay and Schedule-Free admit vanishing
stationarity guarantees under standard nonconvex assumptions. This unified view
then motivates our empirical question: which trade-off is most effective for
large-scale diffusion training?

We therefore turn to the empirical consequences of this trade-off in large-scale
diffusion training. Using 675M-parameter SiT-XL/2 + REPA pretraining, we compare
\textsc{Constant LR + EMA}, \textsc{LR Decay + EMA}, and \textsc{Schedule-Free} under controlled
training branches, asking how each strategy affects generation quality,
optimization progress, and parameter variability.

Three findings emerge. First, LR decay and EMA are complementary
rather than interchangeable: decay alone does not surpass \textsc{Constant LR + EMA}, but
\textsc{Decay + EMA} can further improve the standard recipe, with gains that depend on
when decay begins. Second, Schedule-Free provides more consistent improvements,
outperforming \textsc{Constant LR + EMA} across training stages without requiring a
prespecified decay point. Third, the strongest stabilization is not the best
optimization outcome: \textsc{Decay + EMA} achieves the lowest parameter variability
but weakens late-stage progress, whereas Schedule-Free remains stable at a
larger LR scale and continues to make optimization progress. This allows
it to strike a better balance between stability and continued optimization.

Taken together, these results suggest that EMA is highly effective but not uniquely suited to large-scale diffusion optimization. The broader lesson is that reducing trajectory variability is important but insufficient: effective optimization must also preserve continued progress. This perspective helps explain why Schedule-Free performs well in our main pretraining setting and why its benefits extend to ReFL~\citep{xu2023imagereward} and generator-side DMD~\citep{yin2024dmd,jiang2025dmdr}.

\section{Preliminaries}
\label{sec:theory}

\subsection{Why Constant-Step SGD Has a Noise Floor?}
\label{sec:constant_sgd}

To see why variance reduction is necessary, consider the simplest one-dimensional quadratic $L(x)=\frac{1}{2}dx^2$ with $d>0$, and noisy gradients $g_t=dx_t+\xi_t$, where $\xi_t\sim\mathcal{N}(0,\sigma^2)$ are independent. We use this simple quadratic example to make the mechanism explicit; Lemma~\ref{lem:constant_lr_nonconvex}
below gives the corresponding result for general smooth nonconvex objectives.

Constant-step SGD follows
\begin{equation}
x_{t+1}:=x_t-\eta g_t
=\underbrace{(1-\eta d)x_t}_{\text{contraction toward }0}
-\underbrace{\eta\xi_t}_{\text{fresh noise}}.
\label{eq:constant_sgd}
\end{equation}
The two terms have opposite effects. Without noise, $0<\eta d<2$ makes the iterates converge to the optimum $x^\star=0$. With noise, however, even at the optimum we have $x_{t+1}=-\eta\xi_t\neq0$ in general. Thus each step contracts the previous error while injecting fresh noise. \textbf{Constant-step SGD therefore does not settle at the optimum; it fluctuates around it.}

Because old perturbations are geometrically contracted while fresh noise is injected at every step,
\begin{equation}
x_t=(1-\eta d)^t x_0-\eta\sum_{k=0}^{t-1}(1-\eta d)^{t-1-k}\xi_k,
\qquad
x_t\xrightarrow{d}\mathcal{N}\!\left(0,\frac{\eta\sigma^2}{d(2-\eta d)}\right),
\label{eq:stationary_1d}
\end{equation}
for $0<\eta d<2$. Thus, the initialization vanishes, but the accumulated gradient noise does not: stable constant-step SGD converges to a distribution around the optimum rather than a deterministic point.

On the quadratic objective, this stationary variance directly produces excess loss, as
\begin{equation}
\mathbb{E}[L(x_\infty)]
=\frac{d}{2}\operatorname{Var}(x_\infty)
=\frac{\eta\sigma^2}{2(2-\eta d)}.
\label{eq:noise_floor}
\end{equation}
Both the parameter variance and excess loss increase with the learning rate and grow rapidly as $\eta d$ approaches the stability boundary $2$. \textbf{Variance reduction therefore aims to lower this late-stage noise floor without unnecessarily suppressing further optimization progress.}

The following standard smooth nonconvex SGD bound makes the same noise-floor
phenomenon explicit beyond the quadratic example~\citep[cf.][]{ghadimi2013stochastic,bottou2018optimization}.
\begin{lemma}[Noise floor of constant-step SGD]
\label{lem:constant_lr_nonconvex}
Under the standard smooth nonconvex assumptions in Appendix~\ref{app:nonconvex}, constant-step SGD with learning rate $\eta$ satisfies
\begin{equation}
\frac{1}{T}\sum_{t=0}^{T-1}\mathbb{E}\!\left[\|\nabla f(x_t)\|^2\right]
\leq
\underbrace{\frac{2(f(x_0)-f_\star)}{\eta T}}_{\text{optimization error}}
+\underbrace{L\eta\sigma^2}_{\text{noise floor}}.
\label{eq:nonconvex_bound}
\end{equation}
\end{lemma}
The optimization error vanishes as $T$ grows, but the noise floor remains under a fixed learning rate. Thus, longer training alone cannot eliminate the residual error induced by stochastic gradients.

\subsection{EMA Averages a Noisy Trajectory}
\label{sec:ema_theory}

EMA does not change the online optimization dynamics. Instead, it forms a smoother evaluation point by averaging past iterates
\begin{equation}
x_t^{\mathrm{EMA}}:=\beta x_{t-1}^{\mathrm{EMA}}+(1-\beta)x_t,
\qquad 0<\beta<1.
\end{equation}
For Eq.~\eqref{eq:constant_sgd}, let $b:=1-\eta d$ and define
$m_t:=\beta^t x_0^{\mathrm{EMA}}+(1-\beta)\sum_{j=1}^t
\beta^{t-j}b^j x_0$ and
$w_{k,t}:=\sum_{j=k+1}^t\beta^{t-j}b^{j-1-k}$. Expanding the two
recursions gives the exact EMA iterate
\begin{equation}
x_t^{\mathrm{EMA}}
=\underbrace{m_t}_{\text{mean trajectory / optimization process}}
-\underbrace{\eta(1-\beta)\sum_{k=0}^{t-1}w_{k,t}\xi_k}_{\text{stochastic fluctuation / variance}}.
\label{eq:main_ema_closed_form}
\end{equation}
Thus, the online trajectory continues to fluctuate as before, while EMA averages out part of this fluctuation. For the quadratic model in Section~\ref{sec:constant_sgd}, the stationary variance is reduced by
\begin{equation}
\frac{\operatorname{Var}(x_\infty^{\mathrm{EMA}})}{\operatorname{Var}(x_\infty)}
=\frac{(1-\beta)(1+\beta b)}{(1+\beta)(1-\beta b)}<1,
\qquad b=1-\eta d,
\end{equation}
with the derivation deferred to Appendix~\ref{app:ema}.

Variance reduction, however, does not imply that EMA corrects the optimization bias of the underlying constant-step trajectory.

\begin{lemma}[EMA does not in general remove constant-step optimization bias]
\label{lem:ema_bias}
Consider the smooth strongly convex objective
\[
f(x)=\frac{x^2}{2}+a(\sin x-x),\qquad 0<a<1.
\]
Its unique minimizer is $x^\star=0$. For constant-step SGD
$X_{t+1}=X_t-h\{f'(X_t)+\xi_t\}$ with
$\xi_t\overset{\mathrm{i.i.d.}}{\sim}\mathcal{N}(0,\sigma^2)$,
$\sigma^2>0$, and $0<h\leq1/(1+a)$, let $U$ follow its invariant
distribution. Its mean satisfies
\[
m_h:=\mathbb{E}[U]
=a\bigl(1-\mathbb{E}[\cos U]\bigr)>0=x^\star.
\]
Moreover, if $Y_t=\beta Y_{t-1}+(1-\beta)X_t$ is the \emph{EMA} of a stationary
trajectory, then $\mathbb{E}[Y_t]\to m_h$. Since $f$ is
$(1-a)$-strongly convex,
\[
\mathbb{E}[f(Y_t)]-f(x^\star)
\geq
f(\mathbb{E}[Y_t])-f(x^\star)
\;\longrightarrow\;
f(m_h)-f(x^\star)
\geq
\frac{1-a}{2}m_h^2
>0.
\]
Thus, EMA can reduce fluctuations around the stationary mean, but averaging alone
does not in general eliminate the residual optimization error of the underlying
constant-step trajectory.
\end{lemma}

The proof is given in Appendix~\ref{app:ema_bias}. Hence, EMA improves the stability of the \emph{evaluation weights} without changing the optimization dynamics that generate them.

\subsection{Learning-Rate Decay Shrinks the Noisy Trajectory}
\label{sec:lr_decay_theory}

Unlike EMA, LR decay changes the online optimization dynamics themselves. With a time-varying learning rate $\eta_t$, the quadratic model in Section~\ref{sec:constant_sgd} becomes
\begin{equation}
x_{t+1}:=(1-\eta_t d)x_t-\eta_t\xi_t.
\label{eq:lr_decay_sgd}
\end{equation}
Both the deterministic update and the injected noise are now scaled by $\eta_t$. In particular, the fresh noise has variance $\eta_t^2\sigma^2$, which vanishes as $\eta_t\to0$. Thus, LR decay reduces variance by shrinking the trajectory itself rather than averaging a noisy trajectory after the fact.

For example, consider the analytically tractable schedule
$\eta_t=\eta/(2+t\eta d)$ with $0<\eta d<2$. For every $t\geq1$,
the iterate has the exact form
\begin{equation}
x_t
=\underbrace{\frac{2-\eta d}{2+(t-1)\eta d}x_0}_{\text{mean trajectory / optimization process}}
-\underbrace{\frac{\eta}{2+(t-1)\eta d}
\sum_{k=0}^{t-1}\xi_k}_{\text{stochastic fluctuation / variance}}.
\label{eq:main_decay_iterate}
\end{equation}
Consequently,
\begin{equation}
\mathbb{E}[x_t]=\frac{2-\eta d}{2+(t-1)\eta d}x_0=O(t^{-1}),
\qquad
\operatorname{Var}(x_t)=\frac{t\eta^2\sigma^2}{\bigl(2+(t-1)\eta d\bigr)^2}=O(t^{-1}).
\label{eq:main_decay_closed_form}
\end{equation}
Hence both the optimization error and parameter variance vanish asymptotically. The same mechanism also reveals the central trade-off: reducing $\eta_t$ suppresses stochastic fluctuations, but simultaneously reduces the magnitude of further optimization steps.

The following standard LR-decay SGD guarantee formalizes this behavior;
cf.~\citet{ghadimi2013stochastic,bottou2018optimization}.
\begin{lemma}[Nonconvex convergence under LR decay]
\label{lem:lr_decay_nonconvex}
Under the standard smooth nonconvex assumptions in Appendix~\ref{app:decay_nonconvex}, SGD with learning rates $\{\eta_t\}$ satisfies
\begin{equation}
\frac{\sum_{t=0}^{T-1}\eta_t\,\mathbb{E}\!\left[\|\nabla f(x_t)\|^2\right]}{\sum_{t=0}^{T-1}\eta_t}
\leq
\frac{2\bigl(f(x_0)-f_\star\bigr)}{\sum_{t=0}^{T-1}\eta_t}
+L\sigma^2\frac{\sum_{t=0}^{T-1}\eta_t^2}{\sum_{t=0}^{T-1}\eta_t}.
\label{eq:main_decay_nonconvex}
\end{equation}
If $\sum_t\eta_t=\infty$ and $\sum_t\eta_t^2<\infty$, both terms vanish, yielding convergence to stationarity.
\end{lemma}

The inverse-time schedule above is used only to make the mechanism explicit. In practice, diffusion training often applies decay only near the end of training; our experiments therefore study finite late-stage cosine decays and their sensitivity to the decay starting point.

\subsection{Schedule-Free Combines Trajectory Averaging with Large-Step Optimization}

LR decay reduces variance by shrinking the optimization steps themselves. Schedule-Free instead keeps a constant base learning rate while averaging the optimization trajectory over an increasingly long window~\citep{defazio2024road}. In its SGD form, it maintains
\begin{align}
y_t&:=(1-\beta)z_t+\beta x_t,\\
z_{t+1}&:=z_t-\gamma g(y_t,\xi_t),\\
x_{t+1}&:=\left(1-\frac{1}{t+1}\right)x_t+\frac{1}{t+1}z_{t+1},
\label{eq:schedule_free}
\end{align}
where $z_t$ is the optimization sequence and $x_t$ is the averaged sequence used for evaluation. The key distinction is immediate: $z_t$ continues to take constant-scale optimization steps, while $x_t$ averages over a growing portion of the trajectory.

For the quadratic model in Section~\ref{sec:constant_sgd}, if $a:=1-\gamma d(1-\beta)\in(0,1)$, then
\begin{equation}
\mathbb{E}[x_t]=O(t^{-1}),
\qquad
\operatorname{Var}(x_t)=O(t^{-1}),
\qquad
\mathbb{E}[L(x_t)]=O(t^{-1}),
\end{equation}
with the derivation given in Appendix~\ref{app:schedule_free}. Thus, the averaged sequence becomes increasingly stable even though the optimization sequence continues to use a constant base learning rate.

The following result is restated from \citet[Theorem~4.4]{chae2026understanding},
with minor notation changes for consistency with our presentation.
\begin{lemma}[Nonconvex convergence of Schedule-Free SGD]
\label{lem:schedule_free_nonconvex}
Under the standard smooth nonconvex assumptions stated in Appendix~\ref{app:sf_nonconvex}, \emph{Schedule-Free} SGD admits a vanishing stationarity guarantee for its gradient-evaluation sequence.
\end{lemma}

The precise bound and assumptions are given in Appendix~\ref{app:sf_nonconvex}. Conceptually, Schedule-Free modifies this trade-off differently from LR decay: it reduces variability through trajectory averaging while retaining large-step optimization dynamics.

\section{Variance Reduction in Diffusion Models}
\label{sec:experiments}

\subsection{Experimental Setup: Modern Large-Scale Diffusion Pretraining}

We evaluate whether the distinctions in Section~\ref{sec:theory} persist in a modern large-scale diffusion-training regime. Our main setting uses a class-conditional SiT-XL/2 model with approximately 675M parameters \citep{ma2024sit}, trained on ImageNet-1K at $256\times256$ resolution~\citep{deng2009imagenet} with flow matching~\citep{lipman2023flow} and REPA~\citep{yu2025repa}. This is a strong modern diffusion-transformer recipe rather than a deliberately weak baseline. Our \textsc{Constant LR + EMA} baseline follows the standard REPA optimization recipe, with the same optimizer, learning rate, and EMA configuration; our main deviation is the smaller global batch size used because of computational constraints. Across all main comparisons, the model, data, objective, batch size, and evaluation protocol are fixed; only the optimization recipe is changed.

We compare three strategies:
\begin{itemize}\setlength{\itemsep}{1pt}\setlength{\parskip}{0pt}
\item \textbf{\textsc{Constant LR + EMA}:} AdamW~\citep{loshchilov2019decoupled} with a constant learning rate of $1\times10^{-4}$ and EMA weights for evaluation.
\item \textbf{\textsc{LR Decay + EMA}:} independent continuations from matched late-stage checkpoints of the \textsc{Constant LR + EMA} trajectory, each using the same 40k-step cosine decay while retaining EMA. This isolates the effect of introducing late-stage decay and allows us to test sensitivity to its starting point.
\item \textbf{\textsc{Schedule-Free} AdamW:} Schedule-Free AdamW~\citep{defazio2024road} with a base learning rate of $4\times10^{-4}$.
\end{itemize}

The learning rates are optimizer-specific rather than numerically matched, since Schedule-Free operates in a different effective LR regime from standard AdamW; the corresponding LR controls are reported in Appendix~\ref{app:lr_stability}.

The main trajectories are trained for up to 800k optimization steps with global batch size 128. We evaluate each method at three levels: parameter variability, optimization quality, and generation quality. Parameter variability measures how strongly the trajectory fluctuates; optimization quality is measured by the REPA objective and fixed-evaluation losses; and generation quality is measured by FID~\citep{heusel2017gans}, sFID~\citep{nash2021generating}, and Inception Score~\citep{salimans2016improved}. Together, these measurements distinguish \emph{reducing parameter fluctuations} from \emph{actually obtaining a better optimized and better generating model}. Full training and evaluation details are provided in Appendix~\ref{app:main_config}.

\subsection{LR Decay Can Further Improve \textsc{Constant LR + EMA}, but the Gain Is Timing-Sensitive}

To test whether late-stage LR decay can improve \textsc{Constant LR + EMA}, we independently branch from several late-stage checkpoints of the same trajectory and apply the same 40k-step cosine schedule to each branch. The learning rate decays from $10^{-4}$ to $10^{-6}$, while all other training settings remain unchanged. We evaluate each endpoint using FID-5K with three independent sampling seeds and report the mean.

\textbf{The benefit of decay is timing-sensitive.}
Figure~\ref{fig:lr_decay_fid} shows that adding the same 40k-step cosine decay at different late-stage checkpoints leads to noticeably different outcomes. \textsc{Decay + EMA} improves FID-5K at four of the five tested starting points, but slightly hurts at one, indicating that both the magnitude and even the sign of the gain depend on when decay begins. Full FID-5K, IS, and sFID results are reported in Appendix~\ref{app:main_fid_results} (Table~\ref{tab:main_fid_results}).
\par\vspace{0.5ex}
\textbf{LR decay and EMA provide complementary gains.}
The new comparison also shows that LR decay and EMA are not competing alternatives. Across all five decay branches, applying EMA to the decayed trajectory consistently improves FID-5K over using the corresponding online model alone. For example, at the 360k$\rightarrow$400k branch, FID-5K improves from 9.4895 without EMA to 9.1248 with EMA. This suggests that explicit LR decay and parameter averaging provide complementary benefits and can be combined effectively.
\vspace{-1mm}
\begin{figure}[t]
\vspace{-2mm}
\centering
\includegraphics[width=\linewidth]{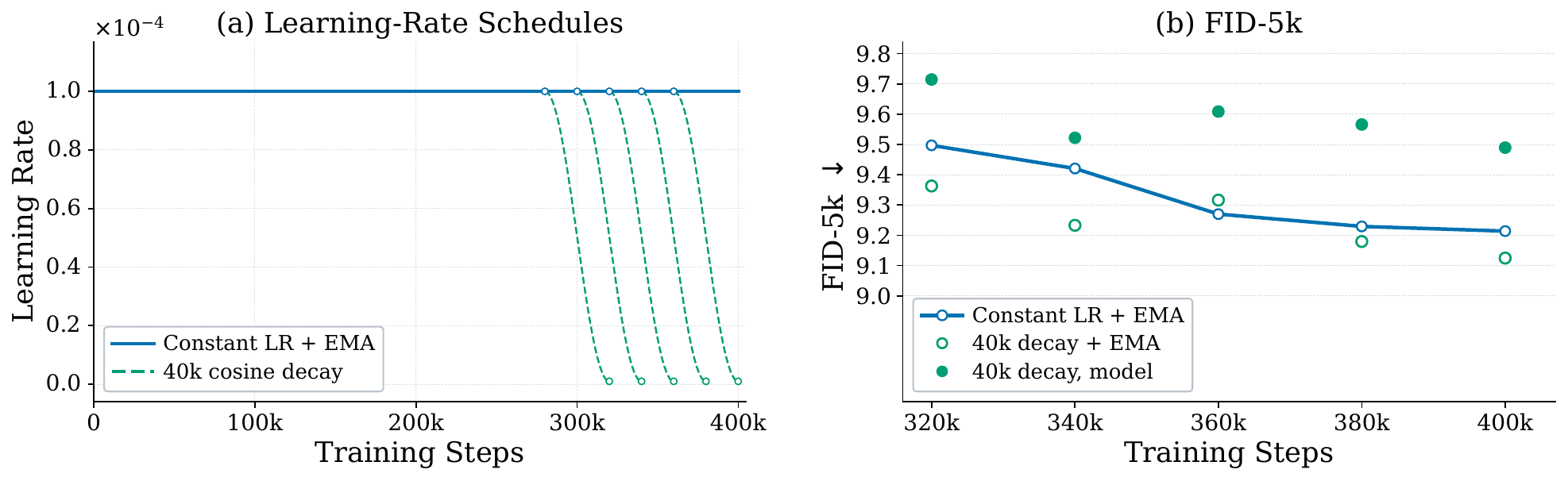}
\caption{Late-stage cosine decay is beneficial but timing-sensitive. (a) Five independent 40k-step cosine-decay branches from the same \textsc{Constant LR + EMA} trajectory. (b) Three-seed mean FID-5K for the matched baseline, decayed online model, and \textsc{Decay + EMA}. EMA further improves all five decay endpoints, while \textsc{Decay + EMA} outperforms the matched baseline at four of five starting points.}
\label{fig:lr_decay_fid}
\vspace{-5mm}
\end{figure}

Thus, late-stage LR decay can improve the standard recipe, but its benefit depends on when annealing begins. We next ask whether Schedule-Free can provide more consistent gains without this timing requirement.

\subsection{Schedule-Free Consistently Improves over \textsc{Constant LR + EMA}}

\textbf{Schedule-Free consistently improves generation quality.}
As shown in Figure~\ref{fig:fid_fixed_eval_composite}(a), \textsc{Schedule-Free AdamW} achieves lower mean FID-5K than \textsc{Constant LR + EMA} at every evaluated checkpoint from 320k to 400k steps, improving FID-5K from 9.2139 to 9.0504 at 400k. While the baseline begins to plateau late in training, Schedule-Free continues to improve, indicating that its advantage is consistent across training stages rather than tied to a single checkpoint.

\begin{figure}[t]
\vspace{-2mm}
\centering
\includegraphics[width=\linewidth]{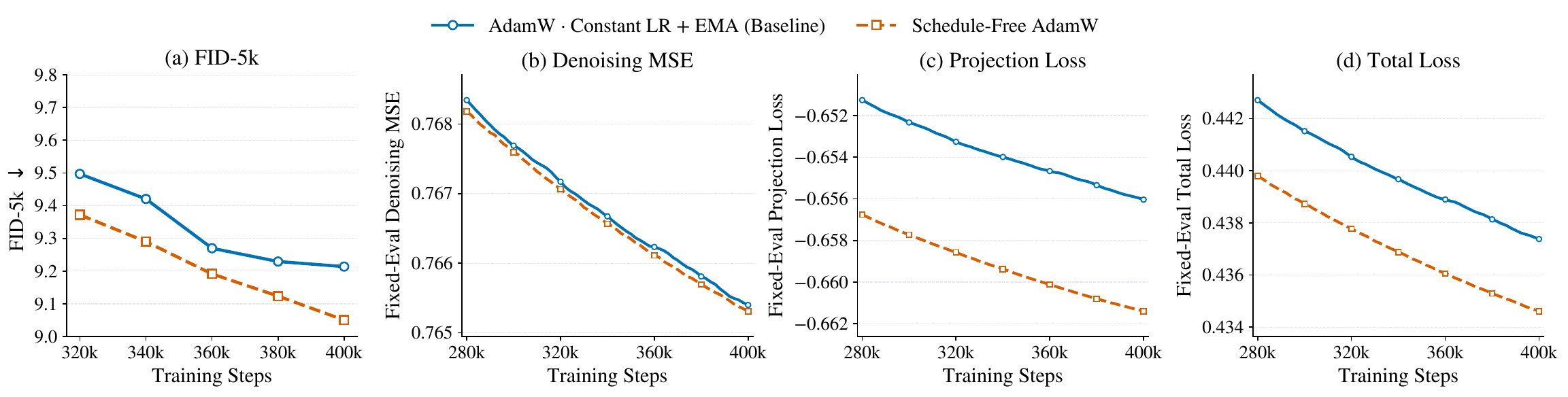}
\caption{Generation quality and fixed-evaluation objectives for \textsc{Constant LR + EMA} and \textsc{Schedule-Free AdamW}. (a) Three-seed mean FID-5K: Schedule-Free is better at every checkpoint from 320k to 400k. (b--d) Fixed-evaluation denoising MSE, projection loss, and total loss: denoising MSE is similar, while Schedule-Free has lower projection and total losses.}
\label{fig:fid_fixed_eval_composite}
\vspace{-3mm}
\end{figure}

\textbf{Schedule-Free also reaches a lower optimization objective.}
Better FID alone does not establish that Schedule-Free reaches a better optimization objective, since parameter averaging could improve sampling quality without reducing the training loss. We therefore evaluate all methods using the same fixed data, noise, class conditions, and timesteps. Under this controlled protocol, Schedule-Free achieves a lower total fixed-evaluation loss throughout the evaluated interval (Figure~\ref{fig:fid_fixed_eval_composite}(b--d)).

Decomposing the objective shows that the denoising MSE remains close between the two methods, while Schedule-Free has a consistently lower projection loss, which accounts for most of the total-loss advantage. To check that the aggregate denoising result does not hide regressions at particular noise levels, we further evaluate fixed timesteps in Appendix~\ref{app:timestep}. Schedule-Free achieves lower denoising MSE at all eight tested timesteps, as well as under random-timestep evaluation.

\textbf{A larger learning rate alone does not explain the gain.}
To control for Schedule-Free's larger base learning rate, we train standard AdamW with the same 5k-step warmup and a constant learning rate of $4\times10^{-4}$. This run becomes unstable and produces sustained NaNs after approximately 218.7k steps, whereas Schedule-Free remains stable at the same learning-rate scale. Thus, the improvement cannot be explained by simply using a larger learning rate with standard AdamW (Appendix~\ref{app:lr_stability}).

\textbf{The advantage persists under longer training.}
At 800k steps, \textsc{Schedule-Free} improves FID-50K from 2.0770 to 1.8735, sFID from 4.539 to 4.479, and Inception Score from 265.6 to 272.7 relative to \textsc{Constant LR + EMA}. Selected samples are shown in Appendix~\ref{app:sf_samples} (Figure~\ref{fig:sf_800k_samples}).

\subsection{Schedule-Free Offers a Better Trade-off than \textsc{Constant LR + EMA} and \textsc{Decay + EMA}}
\label{sec:variance_optimization_tradeoff}

\textbf{Reducing parameter variability helps, but minimizing it is not enough.}
Removing EMA from the \textsc{Constant LR + EMA} baseline substantially degrades FID-5K in our setting. As shown in Appendix~\ref{app:main_fid_results} (Table~\ref{tab:main_fid_results}), the degradation is approximately $1.8$--$4.4$ points across the evaluated checkpoints. In contrast, \textsc{Constant LR + EMA}, \textsc{Decay + EMA}, and \textsc{Schedule-Free AdamW} all substantially reduce parameter variability and achieve much better generation quality. Thus, the relevant question is how to reduce parameter variability without unnecessarily suppressing further optimization progress.

\begin{figure}[H]
\vspace{-3mm}
\begin{minipage}[t]{0.49\linewidth}
\vspace{0pt}
\paragraph{Lower parameter variability does not monotonically improve generation quality.}
Figure~\ref{fig:variability_fid} makes this mismatch explicit. At 400k steps, \textsc{Decay + EMA} has the lowest adjacent-checkpoint variability among the three main strategies, followed by \textsc{Schedule-Free AdamW} and then \textsc{Constant LR + EMA}. Yet the FID-5K ordering is different: \textsc{Schedule-Free AdamW} performs best at 9.0504, compared with 9.1248 for \textsc{Decay + EMA} and 9.2139 for \textsc{Constant LR + EMA}. The same broader pattern holds when the online decay trajectory is included (Appendix~\ref{app:variance}). Thus, stronger trajectory stabilization is beneficial up to a point, but minimizing parameter variability does not by itself produce the best generative model.
\end{minipage}\hfill
\begin{minipage}[t]{0.48\linewidth}
\vspace{0pt}
\centering
\includegraphics[width=\linewidth]{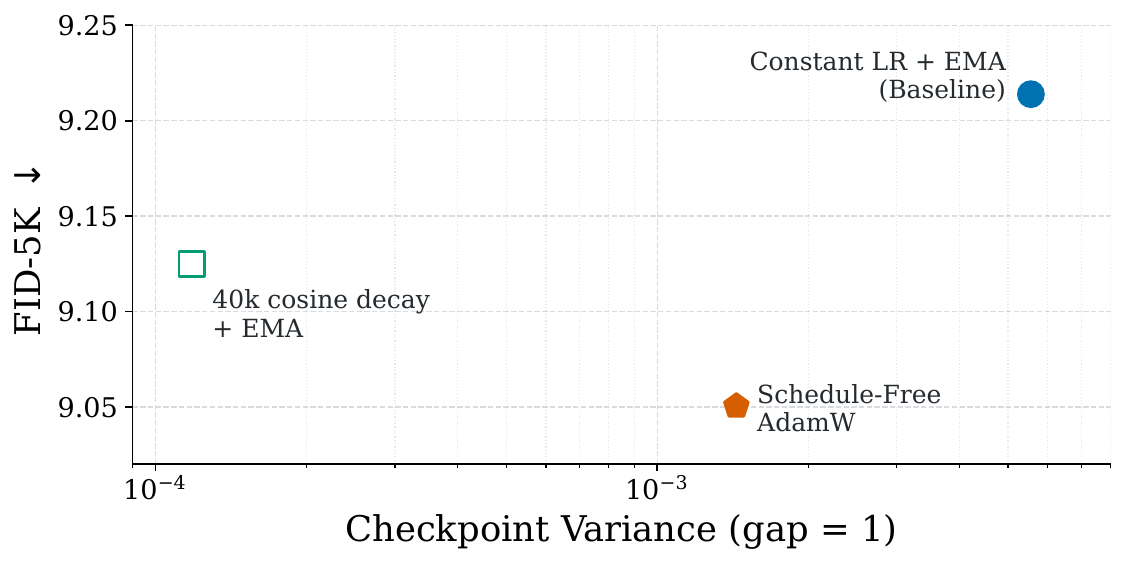}
\caption{Late-stage checkpoint variability versus FID-5K at 400k. Variability is the trace variance between consecutive checkpoints ($g=1$). \textsc{Decay + EMA} is most stable, whereas \textsc{Schedule-Free AdamW} achieves the best FID-5K.}
\label{fig:variability_fid}
\end{minipage}
\vspace{-4mm}
\end{figure}

\textbf{Schedule-Free remains highly stable even at a larger learning-rate scale.}
Notably, the low checkpoint variability of \textsc{Schedule-Free AdamW} in Figure~\ref{fig:variability_fid} is achieved while using a $4\times$ larger base learning rate than the \textsc{Constant LR + EMA} baseline. Despite this more aggressive learning-rate scale, \textsc{Schedule-Free AdamW} still exhibits substantially lower parameter variability than the baseline. Together with the stability control in Appendix~\ref{app:lr_stability}, where standard AdamW becomes unstable at the same learning-rate scale, this suggests that the strong stabilization of \textsc{Schedule-Free} may permit a wider range of stable learning rates. One possible interpretation is that this additional stability margin gives \textsc{Schedule-Free} more room to operate at a larger learning-rate scale, helping preserve continued optimization progress without relying on progressively smaller updates for stability.

\textbf{EMA remains beneficial across training schedules.}
EMA improves generation quality under both constant and decaying learning rates. Even after cosine decay stabilizes the online trajectory, EMA further reduces checkpoint variability and improves FID-5K across all tested branches (Appendix~\ref{app:variance}; Table~\ref{tab:main_fid_results}), confirming that parameter averaging and LR scheduling are complementary rather than interchangeable.

\begin{figure}[t]
\vspace{-4mm}
\centering
\includegraphics[width=\linewidth]{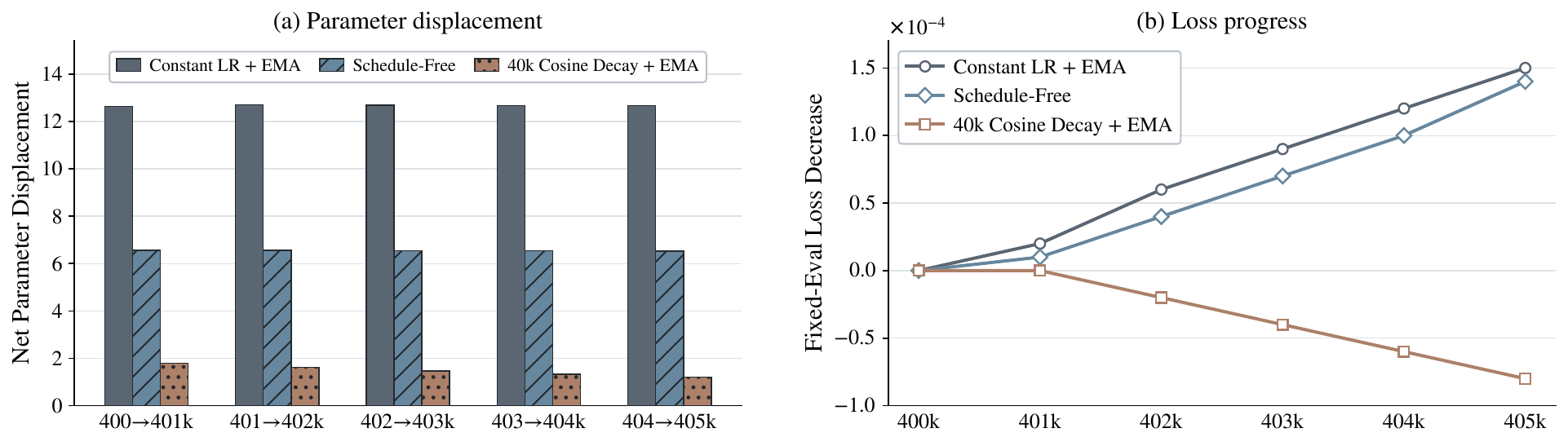}
\caption{Late-stage parameter movement and optimization progress. (a) Net $\ell_2$ displacement over five adjacent 1k-step intervals. (b) Fixed-evaluation loss decrease relative to 400k. \textsc{Decay + EMA} minimizes displacement but does not improve the objective, whereas \textsc{Schedule-Free AdamW} retains continued progress. The comparison uses EMA sequences for constant LR and decay and the averaged $x$ sequence for Schedule-Free; full results are in Appendix~\ref{app:variance}.}
\label{fig:mobility_progress}
\vspace{-7mm}
\end{figure}

\textbf{Aggressive reduction in parameter variability also suppresses optimization mobility.}
The mechanism in Section~\ref{sec:theory} suggests why stronger trajectory stabilization can incur a cost. LR decay scales both the stochastic noise and the deterministic optimization update by $\eta_t$; as $\eta_t$ becomes small, the trajectory becomes more stable but also moves less. Figure~\ref{fig:mobility_progress}(a) shows the corresponding late-stage behavior: across five consecutive 1k-step windows from 400k to 405k, the post-decay EMA trajectory has by far the smallest net parameter displacement, while \textsc{Schedule-Free} retains substantially more movement and \textsc{Constant LR + EMA} moves the most. Thus, the strongest stabilization is also accompanied by the weakest late-stage parameter mobility. Because displacement alone cannot distinguish useful optimization progress from random motion, we next ask whether this movement actually lowers the objective.

\textbf{Schedule-Free better preserves useful late-stage optimization.}
Figure~\ref{fig:mobility_progress}(b) distinguishes useful optimization progress from parameter movement alone. From 400k to 405k, fixed-evaluation loss continues to decrease for both \textsc{Constant LR + EMA} and \textsc{Schedule-Free}, whereas the post-decay EMA run no longer improves and instead increases slightly. Notably, \textsc{Constant LR + EMA} also maintains a higher fixed-evaluation objective than \textsc{Schedule-Free AdamW}, echoing the limitation in Section~\ref{sec:ema_theory} that averaging a constant-step trajectory can reduce parameter variability without in general eliminating residual optimization error. Taken together, Figures~\ref{fig:variability_fid} and~\ref{fig:mobility_progress} show the full trade-off: \textsc{Decay + EMA} produces the strongest stabilization and the least parameter movement, but also the weakest subsequent objective progress, whereas \textsc{Schedule-Free AdamW} substantially reduces variability while continuing to optimize. In our setting, Schedule-Free provides the best empirical balance between these effects.

\subsection{Broader Validation Across Diffusion Training Regimes}

We next test whether the Schedule-Free gains transfer beyond SiT+REPA
pretraining to two distinct diffusion training regimes.

\textbf{ReFL.}
Following the DiffusionOPSD ReFL setup~\citep{xu2023imagereward,zhou2026diffusionopsd}, we replace AdamW with \textsc{Schedule-Free AdamW} while keeping the remaining protocol fixed. AdamW uses the reference LR $1\times10^{-5}$, and \textsc{Schedule-Free AdamW} is evaluated at $4\times$ and $20\times$ this baseline rate. On the directly optimized HPSv3 reward, \textsc{Schedule-Free AdamW} improves the AdamW score from 6.3500 to 6.8626 at $4\times$ LR and to 9.6134 at $20\times$ LR. Cross-reward evaluation is mixed: several orthogonal metrics improve, while others change only slightly or degrade. Complete results are reported in Appendix~\ref{app:refl} (Table~\ref{tab:refl}).

\begin{figure}[t]
\vspace{-4mm}
\centering
\includegraphics[width=0.88\linewidth]{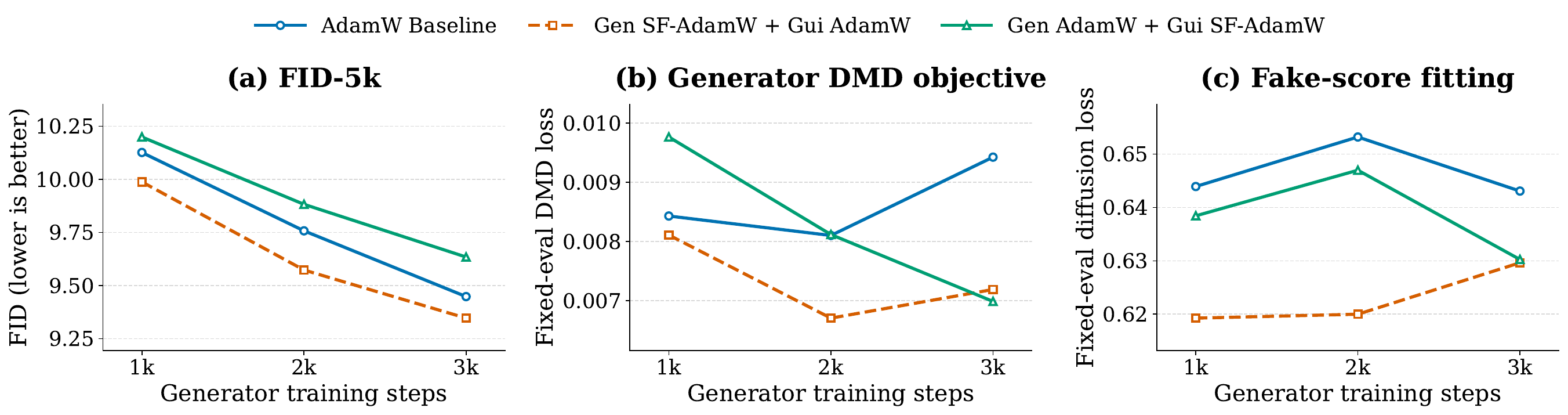}
\caption{DMD optimizer ablation comparing all-AdamW with Schedule-Free applied only to the generator (Gen) or guidance model (Gui). (a) FID-5K: generator-side Schedule-Free improves all checkpoints, whereas guidance-side Schedule-Free does not. (b) Generator DMD objective $\mathcal{L}_{\mathrm{DMD}}$. (c) Fake-score diffusion loss $\mathcal{L}_{\mathrm{diff}}$; definitions are in Appendix~\ref{app:dmd}.}
\label{fig:dmd_combined}
\vspace{-7mm}
\end{figure}

\textbf{DMD.}
DMD~\citep{yin2024dmd} jointly optimizes a generator and a guidance model, with guidance continually tracking the evolving generator distribution. Because the appropriate stopping region is unknown in advance, explicit decay must choose its starting point beforehand, while a fixed EMA timescale requires tuning to the evolving dynamics. Schedule-Free requires no prespecified decay point and continuously provides averaged evaluation parameters, making it natural for this coupled setting.

We compare the all-AdamW one-step DMDR baseline~\citep{jiang2025dmdr} with two variants that apply \textsc{Schedule-Free AdamW} only to the generator or guidance. The replaced component uses $2\times$ its baseline learning rate with a 200-step warmup; all other settings are unchanged (Appendix~\ref{app:dmd}).

Generator-side Schedule-Free improves FID-5K at every evaluated checkpoint and lowers both $\mathcal{L}_{\mathrm{DMD}}$ and $\mathcal{L}_{\mathrm{diff}}$. We hypothesize that Schedule-Free stabilizes the generator optimization trajectory, causing its distribution to evolve more smoothly and thereby presenting the fake-score model with an easier moving target to match. The lower generator objective and fake-score fitting loss in Figure~\ref{fig:dmd_combined}(b--c) are consistent with, but do not establish, this explanation.

The guidance model plays a tracking role, adapting to the evolving fake
distribution rather than directly determining the final evaluation weights.
Improving guidance optimization does not necessarily translate into better
generator samples. Consistent with this distinction, applying \textsc{Schedule-Free} only
to guidance yields worse FID-5K than the all-AdamW baseline at all three
checkpoints. Complete generation metrics are reported in Appendix~\ref{app:dmd}.

\section{Conclusion and Practical Recommendations}

We study \textsc{Constant LR + EMA}, LR decay, and \textsc{Schedule-Free} through a unified
variance-reduction perspective, revealing different trade-offs between parameter
stability and continued optimization. EMA stabilizes through averaging, LR decay
through smaller updates, and \textsc{Schedule-Free} combines strong stabilization with
continued optimization without requiring a prespecified decay point. These differences
lead to the following practical recommendations.

\textbf{When the decay point is uncertain, prefer \textsc{Schedule-Free AdamW}.}
\textsc{Schedule-Free} avoids the need to choose an annealing point in advance. Its strong
stability also makes it worthwhile to explore larger base learning rates than those
typically used with \textsc{Constant LR + EMA}, provided training remains stable. In our main
setting, $4\times10^{-4}$ with 5k-step warmup works well, although other regimes
may require retuning.

\textbf{When the training endpoint is fixed, use decay as a finishing strategy.}
Retaining EMA while annealing the learning rate near the end of training can further
improve the standard \textsc{Constant LR + EMA} recipe. Because the benefit depends on when
annealing begins, decay is most useful when the remaining training budget is known
in advance.

\textbf{Choose for both stability and continued optimization.}
Do not select an optimization strategy solely for the lowest parameter variability.
Prefer a recipe that reduces trajectory fluctuations while still preserving sufficient
late-stage optimization progress.

\subsection*{AI Use Statement}
Generative AI was used for language polishing and stylistic refinement of the manuscript.

\bibliography{references}
\bibliographystyle{iclr2027_conference}

\clearpage

\appendix
\section{Experimental Details}

\subsection{Main Pretraining Configuration}
\label{app:main_config}
All main experiments use the approximately 675M-parameter SiT-XL/2~\citep{ma2024sit} for class-conditional $256\times256$ ImageNet-1K generation~\citep{deng2009imagenet} with a flow-matching + REPA objective~\citep{lipman2023flow,yu2025repa}. The main training trajectories run for up to 800k optimization steps. REPA uses DINOv2 ViT-B features~\citep{oquab2024dinov2} and projection-loss coefficient 0.5. We use no weight decay, global batch size 128, bf16 precision, and gradient clipping with maximum norm 1.0. The dataset, architecture, objective, batch size, precision, and regularization settings are shared across the main comparisons; only the optimizer, learning rate and schedule, and EMA usage differ.

\paragraph{Training objective.}
For a clean image $x$, noise $\epsilon$, diffusion timestep $t$, and class condition $y$, the interpolated input and its flow-matching velocity target are
\begin{equation}
x_t=\alpha_t x+\sigma_t\epsilon,
\qquad
v_t^*=\dot{\alpha}_t x+\dot{\sigma}_t\epsilon,
\end{equation}
where $\alpha_t$ and $\sigma_t$ are the signal and noise coefficients of the interpolation path, and the dots denote derivatives with respect to $t$. The model $v_\theta(x_t,t,y)$ predicts this target velocity. The REPA objective contains two components:
\begin{align}
\mathcal{L}_{\mathrm{MSE}}
&=\operatorname{mean}\!\left[\left\|v_\theta(x_t,t,y)-v_t^*\right\|_2^2\right],
\label{eq:repa_mse}\\
\mathcal{L}_{\mathrm{proj}}
&=-\frac{1}{K}\sum_{k=1}^{K}\operatorname{mean}\!\left[
\left\langle
\frac{z_k}{\|z_k\|_2},
\frac{\tilde z_k}{\|\tilde z_k\|_2}
\right\rangle\right].
\label{eq:repa_proj}
\end{align}
Here, the mean in $\mathcal{L}_{\mathrm{MSE}}$ averages the squared velocity-prediction error over the evaluated samples and tensor elements. In $\mathcal{L}_{\mathrm{proj}}$, $K$ is the number of selected frozen-encoder feature layers, $z_k$ is the target feature from the $k$th selected DINOv2 layer, and $\tilde z_k$ is the corresponding model-projector output. The inner product is computed after $\ell_2$ normalization, so this term is the negative mean cosine similarity. Consequently, lower $\mathcal{L}_{\mathrm{MSE}}$ indicates more accurate flow prediction, while lower $\mathcal{L}_{\mathrm{proj}}$ indicates better representation alignment.

With $\lambda_{\mathrm{proj}}=0.5$, the total objective reported in our loss comparisons is
\begin{equation}
\mathcal{L}_{\mathrm{fixed}}
=\mathcal{L}_{\mathrm{MSE}}+\lambda_{\mathrm{proj}}\mathcal{L}_{\mathrm{proj}}
=\mathcal{L}_{\mathrm{MSE}}+0.5\mathcal{L}_{\mathrm{proj}}.
\label{eq:fixed_eval_objective}
\end{equation}
For fixed-evaluation comparisons, all methods are evaluated using the same data examples, sampled noise, class conditions, and timesteps. Holding these quantities fixed reduces evaluation randomness, so differences in $\mathcal{L}_{\mathrm{MSE}}$, $\mathcal{L}_{\mathrm{proj}}$, and $\mathcal{L}_{\mathrm{fixed}}$ reflect differences between the evaluated model parameters rather than different evaluation draws.

\begin{table}[H]
\caption{Main pretraining optimizer configurations.}
\label{tab:main_config}
\centering
\small
\begin{tabular}{llllll}
\toprule
Method & Optimizer & Base LR & Warmup & LR schedule & EMA\\
\midrule
\textsc{Constant LR + EMA} & AdamW & 1e-4 & None & Constant & 0.9999\\
\textsc{Decay + EMA} & AdamW & 1e-4 & None & cosine to 1e-6 & 0.9999\\
\textsc{Schedule-Free} & SF-AdamW & 4e-4 & 5k linear & Schedule-Free & None\\
\bottomrule
\end{tabular}
\end{table}

\subsection{Learning-Rate Schedules}
\begin{figure}[t]
\centering
\includegraphics[width=\linewidth]{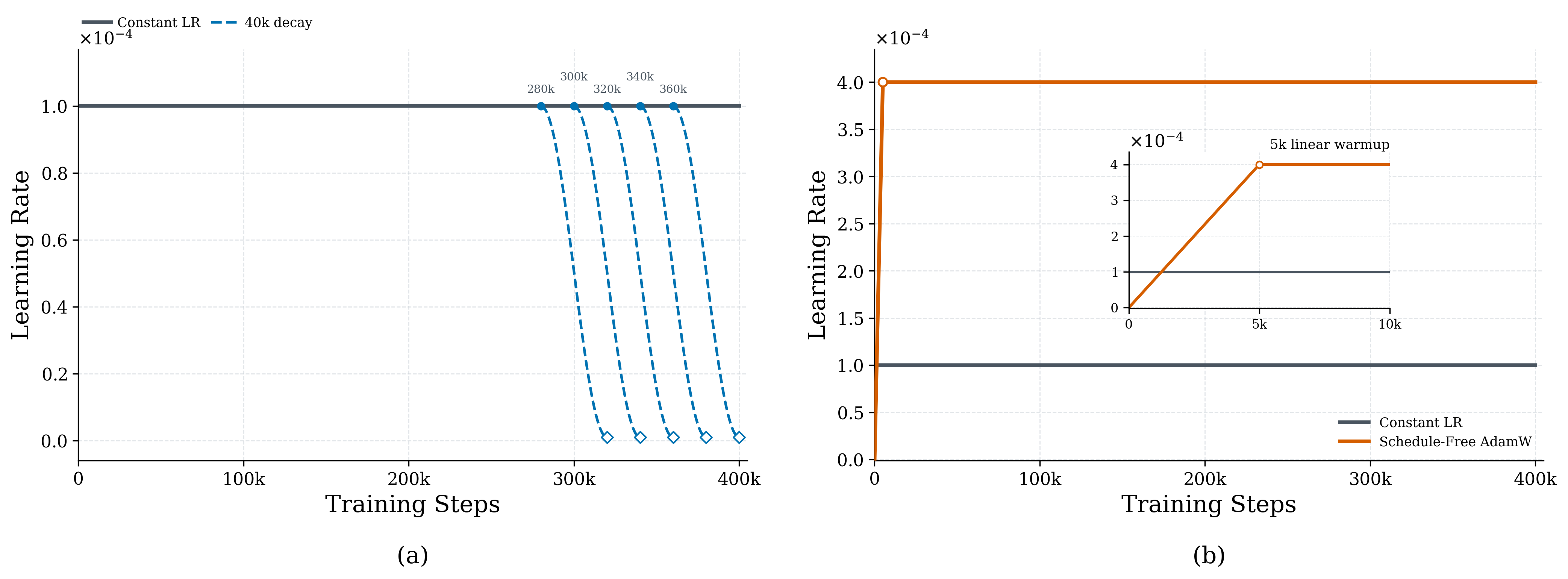}
\caption{Learning-rate schedules used in the main experiment. (a) \textsc{Constant LR + EMA} and independent 40k-step cosine-decay branches from different late-stage checkpoints of the same Constant-LR AdamW trajectory. Each branch decays the learning rate from $10^{-4}$ to $10^{-6}$; different dashed curves indicate different decay starting points. (b) Constant-LR AdamW uses $10^{-4}$ throughout, whereas \textsc{Schedule-Free AdamW} uses a base learning rate of $4\times10^{-4}$ with a 5k-step linear warmup.}
\label{fig:lr_schedules}
\end{figure}

All decay branches share the same trajectory before branch step $t_s$. For $t_s\leq t\leq t_s+T$,
\begin{equation}
\eta_t=\eta_{\min}+\frac12(\eta_0-\eta_{\min})\left[1+\cos\!\left(\pi\frac{t-t_s}{T}\right)\right],
\end{equation}
where $\eta_0=10^{-4}$, $\eta_{\min}=10^{-6}$, and $T=40$k. Main generation results evaluate the endpoint $\theta_{t_s+T}$ directly. Only the checkpoint-variability and late-stage progress diagnostics continue beyond the endpoint at a fixed $10^{-6}$. Schedule-Free warmup is $\eta_t^{\mathrm{SF}}=4\times10^{-4}(t/5000)$ for $0\leq t\leq5000$.

\subsection{Late-Stage Optimization Metrics}
\label{app:variance}
This section gives the complete protocols and numerical results for the
late-stage diagnostics used in Section~\ref{sec:variance_optimization_tradeoff}.
Together, checkpoint variability, net parameter displacement, and
fixed-evaluation loss progress characterize parameter variability, retained
parameter movement, and continued objective improvement.

\subsubsection{Checkpoint Variability}
From each late-stage regime we sample $N=20$ checkpoints $\theta^{(i)}\in\mathbb{R}^D$. For dimension $j$,
\begin{equation}
\bar\theta_j=\frac1N\sum_{i=1}^N\theta_j^{(i)},\qquad
\operatorname{Var}_j=\frac1{N-1}\sum_{i=1}^N(\theta_j^{(i)}-\bar\theta_j)^2,
\end{equation}
and report $V_{\mathrm{total}}=\sum_{j=1}^D\operatorname{Var}_j$. We refer to this trace-variance statistic as \emph{checkpoint variability}. Because it measures finite-window empirical spread, it may include both stochastic fluctuations and systematic movement and is not identified with pure theoretical stochastic variance.

The checkpoint gap $g$ denotes the index spacing between samples in the saved-checkpoint sequence; $g=1$ therefore selects consecutive saved checkpoints. We use 20 such checkpoints, $\theta_t,\ldots,\theta_{t+19}$. This short adjacent-checkpoint window emphasizes local parameter variability without introducing the larger systematic movement accumulated over longer gaps. Constant-LR and Schedule-Free checkpoints come directly from their late-stage trajectories. For cosine decay only, this diagnostic continues from the decay endpoint at a fixed $10^{-6}$ terminal rate.

\begin{table}[H]
\caption{Checkpoint variability over 20 consecutive saved checkpoints (checkpoint gap $g=1$), measured by total trace variance. Parentheses give the ratio to the Constant-LR online trajectory. Both model and EMA sequences are reported for cosine decay.}
\label{tab:variance_gaps}
\centering
\scriptsize
\begin{tabular}{lr}
\toprule
Method / sequence & Checkpoint variability (trace variance, $g=1$)\\
\midrule
Constant-LR online & 7.1611709595e+00 (100\%)\\
Schedule-Free averaged $x$ & 1.4373749495e-03 (0.0201\%)\\
Constant-LR EMA & 5.5516868830e-03 (0.0775\%)\\
Cosine-decay model & 7.3927e-04 (0.0103\%)\\
Cosine-decay EMA & 1.1807e-04 (0.0016\%)\\
\bottomrule
\end{tabular}
\end{table}
At checkpoint gap $g=1$, cosine-decay EMA has the lowest local checkpoint variability, followed by the cosine-decay model, Schedule-Free averaged $x$, Constant-LR EMA, and the Constant-LR online trajectory.

\subsubsection{Late-Stage Net Parameter Displacement}
\label{app:diagnostics}
To directly measure how much parameter change each strategy retains late in training, we compute the $\ell_2$ distance between checkpoints separated by $\tau=1000$ optimization steps:
\begin{equation}
D_\tau(t)=\|\theta_{t+\tau}-\theta_t\|_2.
\end{equation}
We use five consecutive 1k-step windows from 400k to 405k: $400\mathrm{k}\!\to\!401\mathrm{k}$, $401\mathrm{k}\!\to\!402\mathrm{k}$, $402\mathrm{k}\!\to\!403\mathrm{k}$, $403\mathrm{k}\!\to\!404\mathrm{k}$, and $404\mathrm{k}\!\to\!405\mathrm{k}$.

Importantly, $D_\tau(t)$ measures the net change retained between the endpoints of a window, not the path length of the full optimization trajectory within that window. A larger displacement therefore does not by itself imply more effective optimization: it may reflect directed progress, stochastic motion, or both. We pair this measurement with fixed-evaluation loss progress below.

\begin{table}[H]
\caption{Net parameter displacement over five adjacent 1k-step late-stage windows. Each entry is the $\ell_2$ distance $\|\theta_{t+1000}-\theta_t\|_2$, not its square. For the Constant-LR baseline and cosine decay, we report both model and EMA weights; Schedule-Free uses its evaluation-time averaged $x$ sequence.}
\centering
\scriptsize
\resizebox{\linewidth}{!}{\begin{tabular}{llrrrrrr}
\toprule
Run & Weights & 400--401k & 401--402k & 402--403k & 403--404k & 404--405k & Mean\\
\midrule
AdamW baseline & model & 69.953 & 68.033 & 68.007 & 68.192 & 68.026 & 68.442\\
AdamW baseline & EMA & 12.636 & 12.700 & 12.683 & 12.676 & 12.669 & 12.673\\
Schedule-Free & $x$ & 6.5643 & 6.5572 & 6.5492 & 6.5408 & 6.5326 & 6.5488\\
40k cosine decay & model & 0.909307 & 0.843436 & 0.831407 & 0.827545 & 0.820779 & 0.846495\\
40k cosine decay & EMA & 1.794585 & 1.624687 & 1.471350 & 1.333021 & 1.208285 & 1.486386\\
\bottomrule
\end{tabular}}
\end{table}

The Constant-LR online trajectory retains the largest late-stage parameter movement, while EMA substantially reduces it. The Schedule-Free averaged $x$ sequence moves less than Constant LR + EMA but more than both 40k cosine-decay sequences; within cosine decay, the model sequence moves less than EMA. Thus, explicit decay suppresses late-stage parameter mobility most strongly, whereas Schedule-Free retains more movement. Whether this movement corresponds to useful progress is assessed next.

\subsubsection{Fixed-Evaluation Loss Progress}
Net parameter displacement alone cannot determine whether parameter changes represent directed optimization progress. We therefore evaluate every checkpoint with identical data, noise, and timestep samples, reducing evaluation randomness when comparing the 400k--405k trajectories. Define the cumulative loss decrease relative to the 400k checkpoint as
\begin{equation}
\Delta L(t)=L_{400\mathrm{k}}-L_t.
\end{equation}
Here $\Delta L(t)>0$ indicates improvement over the 400k checkpoint, and a larger value indicates more cumulative progress within the late-stage window.
\begin{table}[H]
\caption{Fixed-evaluation total loss from 400k to 405k. All methods use identical fixed evaluation data, noise, and timestep sampling. Constant LR + EMA and Schedule-Free decrease throughout this interval, whereas neither the 40k Cosine Decay model nor its EMA improves in the post-decay regime.}
\centering
\small
\begin{tabular}{lrrrrrr}
\toprule
Method & 400k & 401k & 402k & 403k & 404k & 405k\\
\midrule
Constant LR + EMA & .43738 & .43736 & .43732 & .43729 & .43726 & .43723\\
Schedule-Free & .43461 & .43460 & .43457 & .43454 & .43451 & .43447\\
40k Cosine Decay model & .43696 & .43735 & .43745 & .43754 & .43752 & .43760\\
40k Cosine Decay + EMA & .43691 & .43691 & .43693 & .43695 & .43697 & .43699\\
\bottomrule
\end{tabular}
\end{table}

Constant LR + EMA decreases from $0.43738$ to $0.43723$, and Schedule-Free decreases from $0.43461$ to $0.43447$. In contrast, the 40k Cosine Decay model increases from $0.43696$ to $0.43760$, while its EMA increases slightly from $0.43691$ to $0.43699$. This complements the displacement diagnostic: in the measured post-decay window, explicit decay exhibits both much smaller net movement and substantially weaker subsequent objective progress, whereas Schedule-Free retains more movement while continuing to reduce the objective.

Net displacement and fixed-evaluation progress are therefore complementary diagnostics. The former measures the parameter change retained over a fixed window; the latter tests whether continued training yields objective improvement. Together they support the empirical parameter-variability--optimization trade-off discussed in Section~\ref{sec:variance_optimization_tradeoff}.

\subsection{Learning-Rate Stability Ablation}
\label{app:lr_stability}
We test whether Schedule-Free performs better merely because it uses a larger
base learning rate. The official AdamW configuration trains stably at
$10^{-4}$; to match Schedule-Free's learning-rate scale, we also train AdamW at
$4\times10^{-4}$ with the same 5k-step warmup and otherwise identical settings.
Figure~\ref{fig:lr_stability_ablation} shows that this AdamW run develops a loss
spike after roughly 200k steps and sustained NaNs after approximately 218.7k,
whereas Schedule-Free remains stable. Thus, Schedule-Free can stably use a
larger learning rate that standard AdamW cannot sustain in this setting.

\begin{figure}[t]
\centering
\includegraphics[width=0.72\linewidth]{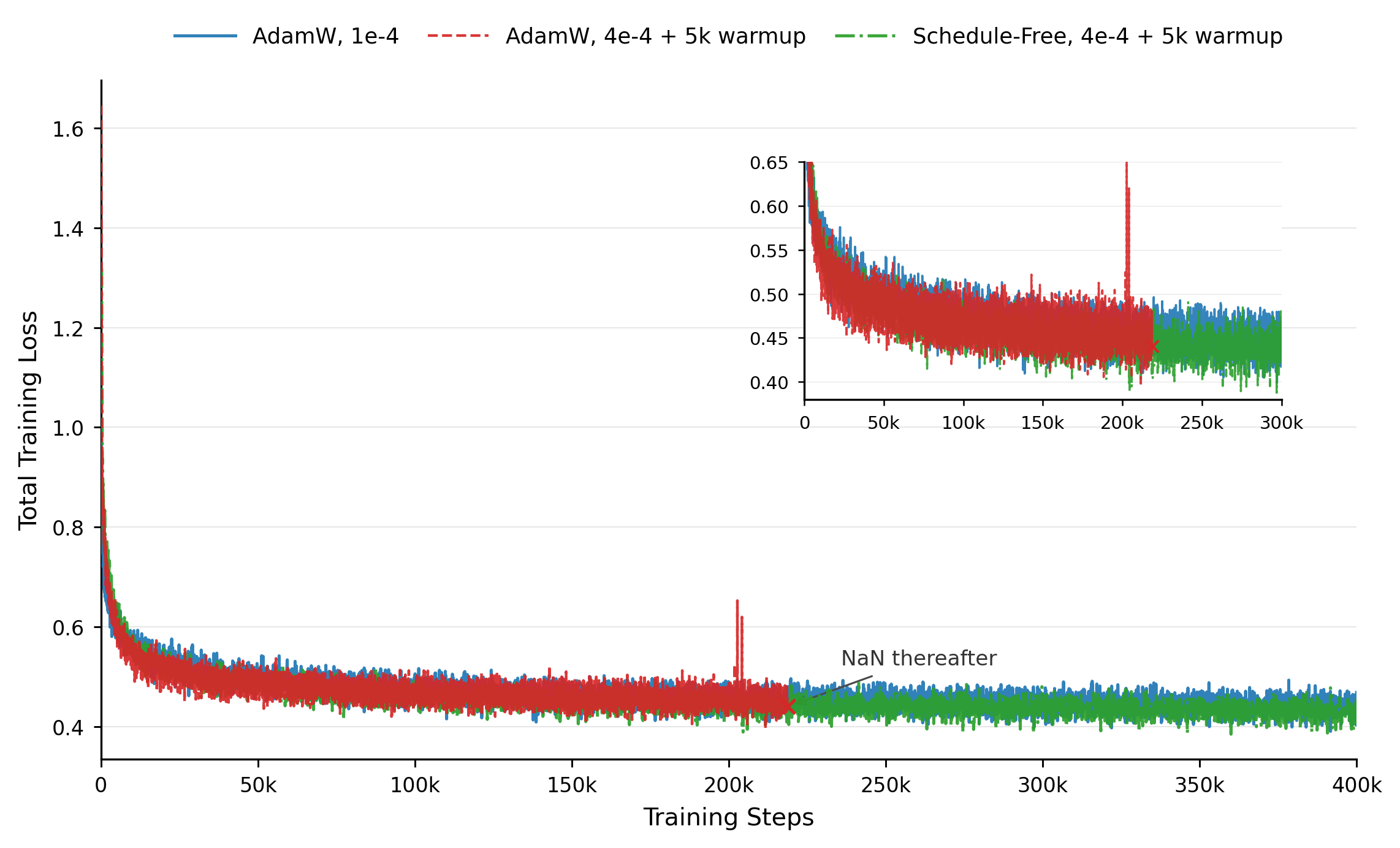}
\caption{Stability at an aggressive learning-rate scale. Standard AdamW with learning rate $4\times10^{-4}$ and 5k-step linear warmup develops a loss spike and sustained NaNs after approximately 218.7k steps. Schedule-Free remains stable with the same rate and warmup; standard AdamW is stable at $10^{-4}$. The inset enlarges the normal-loss range.}
\label{fig:lr_stability_ablation}
\end{figure}

\subsection{ReFL Experimental Details}
\label{app:refl}
Following the ReFL protocol~\citep{xu2023imagereward} in DiffusionOPSD~\citep{zhou2026diffusionopsd}, we run ReFL on Z-Image-Turbo~\citep{zimage2025} and keep the official repository configuration fixed except for optimizer, learning rate, and warmup. HPSv3~\citep{ma2025hpsv3} is the sole training reward. Training uses six policy GPUs and one reward GPU. Each update uses 48 prompt groups and 12 trajectories per prompt ($48\times12=576$ trajectories). Each policy GPU uses micro-batch size 1, synchronizing after accumulating 96 trajectories. Training lasts 100 updates. AdamW uses $\eta_{\mathrm{ref}}=10^{-5}$ without warmup. Schedule-Free uses a five-step warmup and evaluates optimizer $x$ weights; AdamW evaluates EMA LoRA weights~\citep{hu2022lora}. At checkpoint 100, all runs generate from the same 1,000 fixed prompts and seeds at $1024\times1024$, with nine sampling steps, CFG scale 0, bfloat16, and seed 42. HPSv3 is reported as the on-reward metric; PickScore~\citep{kirstain2023pick}, CLIPScore~\citep{hessel2021clipscore}, HPSv2.1~\citep{wu2023human}, Aesthetic Score~\citep{schuhmann2022laion}, ImageReward~\citep{xu2023imagereward}, HPSv3++~\citep{liu2026hpsv3pp}, and the UniPercept IAA, IQA, and ISTA scores~\citep{cao2025unipercept} are used only as orthogonal reward models for cross-reward evaluation.
Qualitative image results for ReFL are provided in Appendix~\ref{app:refl_additional_qualitative}.

\begin{table}[H]
\caption{ReFL results with HPSv3 as the training reward. HPSv3 is the directly optimized on-reward metric; all remaining columns are orthogonal reward models used only for cross-reward evaluation and provide no training signal. Fine-tuned runs are evaluated at checkpoint 100, and the Z-Image base model without reinforcement learning is included as a reference. ``wu0'' and ``wu5'' denote zero and five warmup steps, respectively. Bold indicates the best value in each column among the displayed rows.}
\label{tab:refl}
\centering
\scriptsize
\resizebox{\linewidth}{!}{\begin{tabular}{lrrrrrrrrrr}
\toprule
& \multicolumn{1}{c}{Training reward} & \multicolumn{9}{c}{Orthogonal evaluation rewards}\\
\cmidrule(lr){2-2}\cmidrule(lr){3-11}
Run & HPSv3 & Pick & CLIP & HPSv2.1 & Aes & ImgR & HPSv3++ & Uni IAA & Uni IQA & Uni ISTA\\
\midrule
Z-Image base (no RL) & 6.2544 & 22.8690 & 0.2754 & 0.2962 & 5.4118 & 0.9947 & 12.0641 & \textbf{60.9511} & 75.0656 & 37.7925\\
AdamW $1\times$, wu0 & 6.3500 & 22.8766 & 0.2754 & 0.2971 & 5.4177 & 0.9959 & 12.1271 & 60.7050 & 74.8078 & 38.3209\\
SF-AdamW $4\times$, wu5 & 6.8626 & \textbf{22.8921} & 0.2748 & 0.2995 & 5.4456 & 1.0150 & \textbf{12.3485} & 60.6995 & \textbf{75.2258} & 38.3189\\
SF-AdamW $20\times$, wu5 & \textbf{9.6134} & 22.8341 & \textbf{0.2807} & \textbf{0.3123} & \textbf{5.7066} & \textbf{1.1197} & 12.0347 & 51.5663 & 67.4061 & \textbf{41.7878}\\
\bottomrule
\end{tabular}}
\end{table}

\subsection{DMD Experimental Details}
\label{app:dmd}
We study class-conditional $256\times256$ ImageNet generation~\citep{deng2009imagenet} with SiT-XL/2 initialized from REPA-pretrained weights~\citep{ma2024sit,yu2025repa} in the one-step DMDR distribution-matching setup~\citep{yin2024dmd,jiang2025dmdr}. Guidance is updated once per loop and the generator once per five guidance updates. Reported steps count generator updates. All three configurations train for 3,000 generator steps and are evaluated at 1k, 2k, and 3k. No DINO or reinforcement-learning auxiliary loss is enabled.

The generator is trained with
\begin{equation}
    \mathcal{L}_{\mathrm{DMD}}
    =\frac12\operatorname{mean}\!\left[
      \left(x_0-\operatorname{sg}
      \left(x_0-g_{\mathrm{DMD}}\right)\right)^2\right],
\end{equation}
where $\operatorname{sg}$ denotes stop-gradient, and the DMD direction is
defined as the normalized difference
\begin{equation}
    g_{\mathrm{DMD}}
    =\frac{p_{\mathrm{real}}-p_{\mathrm{fake}}}
      {\operatorname{mean}_{c,h,w}(|p_{\mathrm{real}}|)+\epsilon},
    \qquad
    p_{\mathrm{real}}=x_0-\hat{x}_{0,r},
    \quad
    p_{\mathrm{fake}}=x_0-\hat{x}_{0,f}.
    \label{eq:dmd_guidance}
\end{equation}
Here, $x_0$ is the generator-predicted clean latent, while
$\hat{x}_{0,r}$ and $\hat{x}_{0,f}$ are the clean-latent predictions obtained
from the real and fake estimator modes of the same guidance model, respectively.
The mean in the denominator is computed independently for each sample over
the latent dimensions $(c,h,w)$, rather than over the batch.
% MANUAL CHECK: hats/subscripts in this Word OMML equation were not preserved
% unambiguously by the DOCX converter; compare against the authors' source.
The fake-score model is trained with the diffusion fitting objective
\begin{equation}
    \mathcal{L}_{\mathrm{diff}}
    =\mathbb{E}\left[
      \left\|v_{\mathrm{fake}}(x_t,t,y)-(\epsilon-\tilde{x}_0)\right\|_2^2
      \right].
\end{equation}
Here $\tilde{x}_0\sim G_\theta(z,y)$ denotes a clean latent sampled from the
current generator distribution. It is conceptually the same type of clean
sample as $x_0$ above, but is generated through the separate fake-score
diffusion-fitting computation.
These are also the fixed-evaluation objectives reported in
Figure~\ref{fig:dmd_combined}b--c.

Training uses four GPUs with per-GPU batch size 128 (global 512). In the all-AdamW baseline, both the generator and guidance model use AdamW at $2\times10^{-5}$ without warmup. In the generator-side ablation, the generator instead uses Schedule-Free AdamW at $4\times10^{-5}$ ($2\times$ its baseline learning rate) with a 200-step warmup and averaged $x$ weights, while guidance retains baseline AdamW. In the guidance-side ablation, guidance instead uses Schedule-Free AdamW at $4\times10^{-5}$ ($2\times$ its baseline learning rate) with a 200-step warmup, while the generator retains baseline AdamW. All remaining settings are held fixed. The baseline evaluates online weights. Sampling uses one step, random ImageNet class labels, CFG scale 0, and 5,000 images per checkpoint.

\begin{table}[H]
\caption{DMD generation quality under the all-AdamW baseline and the two component-wise Schedule-Free ablations. Every metric is evaluated on 5,000 generated images. Within each checkpoint, bold marks the best of the three configurations. Lower FID-5K and sFID and higher IS, precision, and recall are better.}
\label{tab:dmd}
\centering
\resizebox{\linewidth}{!}{%
\begin{tabular}{llrrrrr}
\toprule
Method & Ckpt & FID-5K $\downarrow$ & sFID $\downarrow$ & IS $\uparrow$ & Precision $\uparrow$ & Recall $\uparrow$\\
\midrule
AdamW baseline & 1k & 10.126 & \textbf{35.310} & 248.559 & 0.7886 & 0.7395\\
Gen SF-AdamW + Gui AdamW & 1k & \textbf{9.988} & 35.733 & 236.856 & 0.7850 & \textbf{0.7401}\\
Gen AdamW + Gui SF-AdamW & 1k & 10.199 & 35.584 & \textbf{256.577} & \textbf{0.7970} & 0.7210\\
\midrule
AdamW baseline & 2k & 9.757 & \textbf{34.865} & 242.844 & 0.7814 & \textbf{0.7497}\\
Gen SF-AdamW + Gui AdamW & 2k & \textbf{9.573} & 35.261 & \textbf{256.355} & \textbf{0.7950} & 0.7451\\
Gen AdamW + Gui SF-AdamW & 2k & 9.882 & 35.529 & 245.570 & 0.7848 & 0.7395\\
\midrule
AdamW baseline & 3k & 9.448 & \textbf{34.889} & \textbf{258.754} & 0.7932 & 0.7435\\
Gen SF-AdamW + Gui AdamW & 3k & \textbf{9.346} & 35.192 & 251.219 & 0.7890 & \textbf{0.7505}\\
Gen AdamW + Gui SF-AdamW & 3k & 9.634 & 35.477 & 257.485 & \textbf{0.7960} & 0.7449\\
\bottomrule
\end{tabular}
}
\end{table}

\section{Quadratic Analyses and the Classical Nonconvex SGD Bound}
\label{app:quadratic_theory}

\subsection{Constant-LR SGD}
\label{app:constant_sgd}
\subsubsection{Stationary Distribution on Quadratic Objectives}
\begin{theorem}[Stationary distribution of constant-learning-rate SGD]
\label{thm:constant_sgd_stationary}
Consider the one-dimensional quadratic $L(x)=\tfrac12 d x^2$, $d>0$, with stochastic gradient $d x_t+\xi_t$, where the $\xi_t\overset{\mathrm{i.i.d.}}{\sim}\mathcal{N}(0,\sigma^2)$, and deterministic $x_0$. The constant-step recursion is
\begin{equation}
x_{t+1}=(1-\eta d)x_t-\eta\xi_t.
\end{equation}
If $0<\eta d<2$, then
\begin{equation}
x_t\ \Rightarrow\ x_\infty\sim
\mathcal{N}\!\left(0,\frac{\eta\sigma^2}{d(2-\eta d)}\right).
\label{eq:constant_sgd_stationary_variance}
\end{equation}
\end{theorem}

\begin{proof}
Writing $b=1-\eta d$, the first two iterates are
\begin{align*}
x_1&=b x_0-\eta\xi_0,\\
x_2&=b x_1-\eta\xi_1
    =b^2x_0-\eta\bigl(b\xi_0+\xi_1\bigr).
\end{align*}
More generally, if the claimed expression holds at time $t$, then one further
application of the recursion gives
\begin{align*}
x_{t+1}
&=b\left(b^t x_0-\eta\sum_{k=0}^{t-1}b^{t-1-k}\xi_k\right)-\eta\xi_t\\
&=b^{t+1}x_0-\eta\sum_{k=0}^{t-1}b^{t-k}\xi_k-\eta\xi_t\\
&=b^{t+1}x_0-\eta\sum_{k=0}^{t}b^{t-k}\xi_k.
\end{align*}
Thus induction yields
\begin{equation}
x_t=b^t x_0-\eta\sum_{k=0}^{t-1}b^{t-1-k}\xi_k.
\end{equation}
Because $|b|<1$, $b^t x_0\to0$. The stochastic term is a linear
combination of independent centered Gaussian noises and is therefore centered
Gaussian. Independence eliminates all cross-covariances, so its variance is
\begin{align*}
\operatorname{Var}\!\left(
  \eta\sum_{k=0}^{t-1}b^{t-1-k}\xi_k
\right)
&=\eta^2\sum_{k=0}^{t-1}b^{2(t-1-k)}\operatorname{Var}(\xi_k)\\
&=\eta^2\sigma^2\sum_{j=0}^{t-1}b^{2j}
 =\eta^2\sigma^2\frac{1-b^{2t}}{1-b^2}.
\end{align*}
Consequently,
\begin{equation}
\eta^2\sigma^2\sum_{j=0}^{t-1}b^{2j}
\longrightarrow \frac{\eta^2\sigma^2}{1-b^2}
=\frac{\eta\sigma^2}{d(2-\eta d)},
\end{equation}
where the last equality uses
$1-b^2=1-(1-\eta d)^2=\eta d(2-\eta d)$. This proves
Eq.~\eqref{eq:constant_sgd_stationary_variance}.
\end{proof}

\subsubsection{Smooth Nonconvex Stationarity}
\label{app:nonconvex}
\begin{theorem}[Convergence of SGD for smooth nonconvex optimization]
\label{thm:nonconvex_sgd}
Let $f:\mathbb{R}^d\to\mathbb{R}$ be $L$-smooth and bounded below by $f_\star$. Suppose $x_{t+1}=x_t-\eta g_t$, $\mathbb{E}[g_t\mid x_t]=\nabla f(x_t)$, and
\begin{equation}
\mathbb{E}[\|g_t-\nabla f(x_t)\|^2\mid x_t]\leq\sigma^2.
\end{equation}
For $0<\eta\leq1/L$,
\begin{equation}
\frac{1}{T}\sum_{t=0}^{T-1}\mathbb{E}\|\nabla f(x_t)\|^2
\leq \frac{2(f(x_0)-f_\star)}{\eta T}+L\eta\sigma^2.
\label{eq:appendix_nonconvex_bound}
\end{equation}
Choosing $\eta=\Theta(T^{-1/2})$ gives an average squared-gradient norm of $O(T^{-1/2})$, and the same order holds for the best iterate among $t<T$.
\end{theorem}

\begin{proof}
Smoothness gives
\begin{equation}
f(x_{t+1})\leq f(x_t)-\eta\langle\nabla f(x_t),g_t\rangle+\frac{L\eta^2}{2}\|g_t\|^2.
\label{eq:nonconvex_smoothness_step}
\end{equation}
To evaluate the second moment of the stochastic gradient, write
$\zeta_t:=g_t-\nabla f(x_t)$. Conditional on $x_t$, the assumptions give
$\mathbb{E}[\zeta_t\mid x_t]=0$ and
$\mathbb{E}[\|\zeta_t\|^2\mid x_t]\leq\sigma^2$. Hence
\begin{align*}
\mathbb{E}[\|g_t\|^2\mid x_t]
&=\mathbb{E}[\|\nabla f(x_t)+\zeta_t\|^2\mid x_t]\\
&=\|\nabla f(x_t)\|^2
  +2\left\langle\nabla f(x_t),
  \mathbb{E}[\zeta_t\mid x_t]\right\rangle
  +\mathbb{E}[\|\zeta_t\|^2\mid x_t],
\end{align*}
and therefore
\begin{equation}
\mathbb{E}[\|g_t\|^2\mid x_t]
\leq\|\nabla f(x_t)\|^2+\sigma^2.
\end{equation}
Taking conditional expectation in Eq.~\eqref{eq:nonconvex_smoothness_step}, using
$\mathbb{E}[g_t\mid x_t]=\nabla f(x_t)$, and then applying the preceding
second-moment bound gives
\begin{align*}
\mathbb{E}[f(x_{t+1})\mid x_t]
&\leq f(x_t)
-\eta\left\langle\nabla f(x_t),
\mathbb{E}[g_t\mid x_t]\right\rangle
+\frac{L\eta^2}{2}\mathbb{E}[\|g_t\|^2\mid x_t]\\
&\leq f(x_t)-\eta\|\nabla f(x_t)\|^2
+\frac{L\eta^2}{2}
 \bigl(\|\nabla f(x_t)\|^2+\sigma^2\bigr),
\end{align*}
which, after collecting the squared-gradient terms, becomes
\begin{equation}
\mathbb{E}[f(x_{t+1})\mid x_t]
\leq f(x_t)-\left(\eta-\frac{L\eta^2}{2}\right)\|\nabla f(x_t)\|^2+\frac{L\eta^2\sigma^2}{2}.
\end{equation}
For $\eta\leq1/L$,
$\eta-L\eta^2/2\geq\eta/2$. Taking total expectations and rearranging
therefore yields
\begin{equation*}
\frac{\eta}{2}\mathbb{E}\|\nabla f(x_t)\|^2
\leq
\mathbb{E}f(x_t)-\mathbb{E}f(x_{t+1})
+\frac{L\eta^2\sigma^2}{2}.
\end{equation*}
Summing this inequality from $t=0$ to $T-1$ telescopes the function values:
\begin{align*}
\frac{\eta}{2}\sum_{t=0}^{T-1}
\mathbb{E}\|\nabla f(x_t)\|^2
&\leq f(x_0)-\mathbb{E}f(x_T)
+\frac{TL\eta^2\sigma^2}{2}\\
&\leq f(x_0)-f_\star
+\frac{TL\eta^2\sigma^2}{2}.
\end{align*}
Dividing by $\eta T/2$ proves
Eq.~\eqref{eq:appendix_nonconvex_bound}, which is the bound stated in
Eq.~\eqref{eq:nonconvex_bound}. The best-iterate statement follows from
$\min_{0\leq t<T}a_t\leq T^{-1}\sum_{t=0}^{T-1}a_t$ applied to
$a_t=\mathbb{E}\|\nabla f(x_t)\|^2$.
\end{proof}

\subsection{Constant-LR SGD + EMA}
\label{app:ema}
\subsubsection{Stationary Variance Reduction}
\begin{theorem}[EMA reduces stationary variance]
\label{thm:ema_stationary_variance}
Under the one-dimensional quadratic setting of Appendix~\ref{app:constant_sgd}, let $b=1-\eta d$, $|b|<1$, and let $y_t=\beta y_{t-1}+(1-\beta)x_t$ with $0<\beta<1$. The stationary variances satisfy
\begin{equation}
\frac{\operatorname{Var}(y_\infty)}
{\operatorname{Var}(x_\infty)}
=\frac{1-\beta}{1+\beta}
\frac{1+\beta b}{1-\beta b}<1.
\label{eq:ema_stationary_variance_ratio}
\end{equation}
Thus EMA reduces the stationary variance of the evaluation iterate.
\end{theorem}

\begin{proof}
Online SGD is $x_{t+1}=b x_t-\eta\xi_t$ and has stationary variance
\begin{equation}
v:=\operatorname{Var}(x_t)=\frac{\eta^2\sigma^2}{1-b^2}.
\end{equation}

We first justify that the stationary EMA process and its covariance with
$x_t$ are well defined. Take the online iterate in its stationary regime and
define the causal EMA solution
\begin{equation*}
y_t:=(1-\beta)\sum_{j=0}^{\infty}\beta^j x_{t-j}.
\end{equation*}
Because $\mathbb{E}[x_t^2]=v<\infty$ and $0<\beta<1$, its partial sums are
Cauchy in $L^2$: for $m>n$,
\begin{align*}
\left\|(1-\beta)\sum_{j=n+1}^{m}\beta^j x_{t-j}\right\|_{L^2}
&\leq(1-\beta)\sum_{j=n+1}^{m}\beta^j
\|x_{t-j}\|_{L^2}\\
&\leq \sqrt{v}\,\beta^{n+1}\longrightarrow0.
\end{align*}
Thus the series converges in $L^2$, has finite variance, is jointly
stationary with $x_t$, and satisfies
$y_t=\beta y_{t-1}+(1-\beta)x_t$. It is also the unique stationary
$L^2$ solution. Indeed, if $y_t$ and $\widetilde y_t$ are two such solutions,
then $D_t:=y_t-\widetilde y_t=\beta D_{t-1}$; stationarity implies
$\mathbb{E}[D_t^2]=\beta^2\mathbb{E}[D_t^2]$, and hence $D_t=0$ almost surely.
In particular,
$c:=\operatorname{Cov}(y_t,x_t)$ is finite by Cauchy--Schwarz.

Both stationary means are zero. Using
$x_t=b x_{t-1}-\eta\xi_{t-1}$ and expanding the covariance gives
\begin{align*}
c
&=\operatorname{Cov}\!\left(
  \beta y_{t-1}+(1-\beta)x_t,\,x_t
  \right)\\
&=\beta\operatorname{Cov}(y_{t-1},x_t)
  +(1-\beta)\operatorname{Var}(x_t)\\
&=\beta\operatorname{Cov}\!\left(
  y_{t-1},b x_{t-1}-\eta\xi_{t-1}
  \right)+(1-\beta)v\\
&=\beta b\operatorname{Cov}(y_{t-1},x_{t-1})
  -\beta\eta\operatorname{Cov}(y_{t-1},\xi_{t-1})
  +(1-\beta)v.
\end{align*}
Here $y_{t-1}$ is measurable with respect to the history before
$\xi_{t-1}$ is sampled, so
$\operatorname{Cov}(y_{t-1},\xi_{t-1})=0$. Joint stationarity gives
$\operatorname{Cov}(y_{t-1},x_{t-1})=c$, and hence
\begin{equation}
c=\beta b c+(1-\beta)v=\frac{1-\beta}{1-\beta b}v.
\end{equation}
The division is valid because $|\beta b|<1$.

Writing $w=\operatorname{Var}(y_t)$,
\begin{equation}
w=\beta^2w+(1-\beta)^2v+2\beta(1-\beta)bc,
\end{equation}
so Eq.~\eqref{eq:ema_stationary_variance_ratio} follows from
\begin{equation}
\frac{w}{v}=\frac{1-\beta}{1+\beta}\frac{1+\beta b}{1-\beta b}<1,
\end{equation}
because $(1+\beta)(1-\beta b)-(1-\beta)(1+\beta b)=2\beta(1-b)>0$.
\end{proof}

\subsubsection{A Limitation of EMA: Persistent Optimization Bias}
\label{app:ema_bias}

The preceding quadratic model is symmetric, so its invariant distribution is
centered at the minimizer. This need not hold for a nonquadratic objective.
Consider the one-dimensional strongly convex and smooth example
\begin{equation}
    f(x)=\frac{x^2}{2}+a(\sin x-x),
    \qquad 0<a<1,
    \label{eq:ema_bias_objective}
\end{equation}
whose unique minimizer is $x^\star=0$. Constant-step SGD follows
\begin{equation}
    X_{t+1}=F_h(X_t)-h\xi_t,
    \qquad F_h(x):=x-hf'(x),
    \qquad \xi_t\overset{\mathrm{i.i.d.}}{\sim}\mathcal{N}(0,\sigma^2),
    \label{eq:ema_bias_sgd}
\end{equation}
with $\sigma^2>0$.

\begin{theorem}[EMA preserves the stationary optimization bias]
\label{thm:ema_bias}
For the recursion in Eq.~\eqref{eq:ema_bias_sgd}, suppose
$0<h\leq1/(1+a)$. Then SGD has a unique invariant distribution
$\pi_h$ with finite second moment, and the law of $X_t$ converges to
$\pi_h$ geometrically in the $2$-Wasserstein distance. If
$U\sim\pi_h$ and $m_h:=\mathbb{E}[U]$, then
\begin{equation}
    \mathbb{E}[U^2]
    \leq
    \frac{h\sigma^2}
    {(1-a)\{2-h(1-a)\}},
    \qquad
    m_h=a\bigl(1-\mathbb{E}[\cos U]\bigr)>0=x^\star.
    \label{eq:ema_bias_stationary_mean}
\end{equation}
If $X_t$ is stationary and
$Y_t=\beta Y_{t-1}+(1-\beta)X_t$ for $0<\beta<1$, then
\begin{equation}
    \mathbb{E}[Y_t]
    =\beta^t\mathbb{E}[Y_0]+(1-\beta^t)m_h
    \longrightarrow m_h.
    \label{eq:ema_bias_same_mean}
\end{equation}
Thus EMA can reduce the variance around the stationary mean, but it does not
move that mean to the optimizer.
\end{theorem}

\begin{proof}
Since $f''(x)=1-a\sin x$, we have
$1-a\leq f''(x)\leq1+a$. For
$0<h\leq1/(1+a)$,
\begin{equation}
    0\leq F_h'(x)=1-hf''(x)\leq q_h:=1-h(1-a)<1.
\end{equation}
The mean-value theorem therefore gives
$|F_h(x)-F_h(z)|\leq q_h|x-z|$. Couple two SGD chains by using the same
noise at every step. Their noise terms cancel, so their one-step laws satisfy
\begin{equation}
    W_2(\nu_1P_h,\nu_2P_h)
    \leq q_h W_2(\nu_1,\nu_2),
    \label{eq:ema_bias_wasserstein_contraction}
\end{equation}
where $P_h$ is the transition kernel of Eq.~\eqref{eq:ema_bias_sgd}.
The kernel preserves finite second moments. Hence it is a contraction on the
complete metric space of probability measures with finite second moment, and
the Banach fixed-point theorem gives a unique invariant distribution
$\pi_h$ and geometric convergence to it.

Because $F_h(0)=0$, invariance, zero-mean noise, and the contraction imply
\begin{align}
    \mathbb{E}[U^2]
    &=\mathbb{E}\bigl[(F_h(U)-h\xi)^2\bigr]\\
    &=\mathbb{E}[F_h(U)^2]+h^2\sigma^2
    \leq q_h^2\mathbb{E}[U^2]+h^2\sigma^2.
\end{align}
Moving the $q_h^2\mathbb{E}[U^2]$ term to the left gives
\begin{equation*}
    (1-q_h^2)\mathbb{E}[U^2]\leq h^2\sigma^2.
\end{equation*}
From $q_h=1-h(1-a)$,
\begin{align*}
1-q_h^2
&=(1-q_h)(1+q_h)\\
&=h(1-a)\bigl\{2-h(1-a)\bigr\}.
\end{align*}
This quantity is positive because $0<q_h<1$. Dividing by it therefore yields
\begin{align*}
\mathbb{E}[U^2]
&\leq \frac{h^2\sigma^2}{1-q_h^2}\\
&=\frac{h\sigma^2}
{(1-a)\{2-h(1-a)\}},
\end{align*}
which is the second-moment bound in
Eq.~\eqref{eq:ema_bias_stationary_mean}.

To derive the stationary mean, let $U\sim\pi_h$ be independent of a fresh
$\xi\sim\mathcal{N}(0,\sigma^2)$. Invariance means that
$U':=F_h(U)-h\xi$ also has law $\pi_h$, so
$\mathbb{E}[U']=\mathbb{E}[U]=m_h$. Taking expectation in the update gives
\begin{align*}
m_h
&=\mathbb{E}[F_h(U)-h\xi]\\
&=\mathbb{E}[U-hf'(U)]-h\mathbb{E}[\xi]\\
&=m_h-h\mathbb{E}[f'(U)].
\end{align*}
Since $h>0$, subtracting $m_h$ from both sides and dividing by $h$ yields
$\mathbb{E}[f'(U)]=0$. Using
$f'(x)=x+a(\cos x-1)$, we obtain
\begin{align*}
0
&=\mathbb{E}[U]+a\bigl(\mathbb{E}[\cos U]-1\bigr)\\
&=m_h+a\bigl(\mathbb{E}[\cos U]-1\bigr),
\end{align*}
and hence
\begin{equation}
    m_h=a\bigl(1-\mathbb{E}[\cos U]\bigr)\geq0.
\end{equation}
The inequality is strict: conditionally on $U$, the next iterate has a
nondegenerate Gaussian density, so invariance implies that $\pi_h$ is not
supported on the countable set $2\pi\mathbb{Z}$ where $\cos U=1$.
Therefore $\mathbb{E}[\cos U]<1$ and $m_h>0=x^\star$.

Finally, under stationarity of $X_t$, let
$\mu_t:=\mathbb{E}[Y_t]$ and note that
$\mathbb{E}[X_t]=m_h$ for every $t$. Taking expectation in the EMA recursion
gives
\begin{equation*}
    \mu_t=\beta\mu_{t-1}+(1-\beta)m_h.
\end{equation*}
Subtracting $m_h$ from both sides and iterating yields
\begin{align*}
\mu_t-m_h
&=\beta(\mu_{t-1}-m_h)
=\cdots
=\beta^t(\mu_0-m_h),\\
\mu_t
&=\beta^t\mu_0+(1-\beta^t)m_h
\longrightarrow m_h,
\end{align*}
because $0<\beta<1$. This is Eq.~\eqref{eq:ema_bias_same_mean}, so the EMA
sequence inherits the same stationary mean bias.
\end{proof}

\subsection{Learning-Rate Decay SGD + EMA}
\label{app:lr_decay}
\subsubsection{Exact Dynamics under Learning-Rate Decay and EMA}
\label{app:decay_with_ema}

Consider the one-dimensional quadratic model from
Section~\ref{sec:lr_decay_theory} under the inverse-time schedule
$\eta_t=\eta/(2+t\eta d)$. For compactness, define
\begin{equation}
    a:=\eta d,
    \qquad D_t:=2+(t-1)a.
    \label{eq:decay-notation}
\end{equation}

\begin{theorem}[Exact bias--variance decomposition for learning-rate decay]
\label{thm:decay-exact}
Let $f(x)=\tfrac12 d x^2$, $d>0$, and
$g_t=d x_t+\xi_t$, where the noises are independent with
$\mathbb{E}[\xi_t]=0$ and $\mathbb{E}[\xi_t^2]=\sigma^2$.
Suppose $x_0$ is deterministic, $0<\eta d<2$, and
$x_{t+1}=x_t-\eta_t g_t$ with
$\eta_t=\eta/(2+t\eta d)$. Then, for every $t\geq1$,
\begin{equation}
    x_t
    =\underbrace{\frac{2-a}{D_t}x_0}_{\text{mean trajectory}}
    -\underbrace{\frac{\eta}{D_t}
      \sum_{k=0}^{t-1}\xi_k}_{\text{stochastic fluctuation}}.
    \label{eq:decay-exact}
\end{equation}
Consequently,
\begin{equation}
    \mathbb{E}[x_t]=\frac{2-a}{D_t}x_0=O(t^{-1}),
    \qquad
    \operatorname{Var}(x_t)
    =\frac{t\eta^2\sigma^2}{D_t^2}=O(t^{-1}),
    \label{eq:decay-exact-moments}
\end{equation}
and $\mathbb{E}[f(x_t)]=O(t^{-1})$.
\end{theorem}

\begin{proof}
Since $D_{t+1}=2+ta$, $a=\eta d$, and
$\eta_t=\eta/(2+ta)$, we have
\begin{align*}
1-\eta_t d
&=1-\frac{\eta d}{2+ta}
 =\frac{2+ta-a}{2+ta}
 =\frac{2+(t-1)a}{2+ta}
 =\frac{D_t}{D_{t+1}},\\
\eta_t&=\frac{\eta}{D_{t+1}}.
\end{align*}
Substituting these identities into
$x_{t+1}=(1-\eta_t d)x_t-\eta_t\xi_t$ gives
\begin{equation}
    x_{t+1}=\frac{D_t}{D_{t+1}}x_t
    -\frac{\eta}{D_{t+1}}\xi_t.
    \label{eq:decay-rescaled-update}
\end{equation}
Multiplying Eq.~\eqref{eq:decay-rescaled-update} by $D_{t+1}$ and moving
$D_t x_t$ to the left gives the difference identity
\begin{equation*}
    D_{t+1}x_{t+1}-D_t x_t=-\eta\xi_t.
\end{equation*}
Applying this identity at indices $k=0,\ldots,t-1$ and summing yields
\begin{align*}
\sum_{k=0}^{t-1}
  \bigl(D_{k+1}x_{k+1}-D_kx_k\bigr)
&=-\eta\sum_{k=0}^{t-1}\xi_k,\\
\underbrace{(D_1x_1-D_0x_0)
 +(D_2x_2-D_1x_1)+\cdots
 +(D_tx_t-D_{t-1}x_{t-1})}_{D_tx_t-D_0x_0}
&=-\eta\sum_{k=0}^{t-1}\xi_k.
\end{align*}
All intermediate terms cancel. Since $D_0=2-a$, rearranging gives
\begin{equation*}
D_tx_t=(2-a)x_0-\eta\sum_{k=0}^{t-1}\xi_k,
\qquad
x_t=\frac{2-a}{D_t}x_0
-\frac{\eta}{D_t}\sum_{k=0}^{t-1}\xi_k,
\end{equation*}
which is Eq.~\eqref{eq:decay-exact}.

Because $x_0$ is deterministic and every noise variable is centered,
linearity of expectation gives
\begin{align*}
\mathbb{E}[x_t]
&=\frac{2-a}{D_t}x_0
-\frac{\eta}{D_t}\sum_{k=0}^{t-1}\mathbb{E}[\xi_k]
=\frac{2-a}{D_t}x_0.
\end{align*}
The deterministic term contributes no variance, so
\begin{equation*}
\operatorname{Var}(x_t)
=\frac{\eta^2}{D_t^2}
\operatorname{Var}\!\left(\sum_{k=0}^{t-1}\xi_k\right).
\end{equation*}
To evaluate the remaining variance, expand the centered square:
\begin{align*}
\operatorname{Var}\!\left(\sum_{k=0}^{t-1}\xi_k\right)
&=\mathbb{E}\!\left[\left(\sum_{k=0}^{t-1}\xi_k\right)^2\right]\\
&=\sum_{k=0}^{t-1}\mathbb{E}[\xi_k^2]
  +2\sum_{0\leq k<\ell\leq t-1}\mathbb{E}[\xi_k\xi_\ell]\\
&=t\sigma^2.
\end{align*}
The last equality follows because independence and zero means imply
$\mathbb{E}[\xi_k\xi_\ell]
=\mathbb{E}[\xi_k]\mathbb{E}[\xi_\ell]=0$ for $k\neq\ell$.
Substitution proves Eq.~\eqref{eq:decay-exact-moments}. Finally,
\begin{align*}
\mathbb{E}[f(x_t)]
&=\frac d2\mathbb{E}[x_t^2]
=\frac d2\left\{\mathbb{E}[x_t]^2+\operatorname{Var}(x_t)\right\}\\
&=\frac{d(2-a)^2x_0^2}{2D_t^2}
+\frac{dt\eta^2\sigma^2}{2D_t^2}
=O(t^{-1}),
\end{align*}
because the two displayed terms are respectively $O(t^{-2})$ and
$O(t^{-1})$.
\end{proof}

Now define the EMA iterate
\begin{equation}
    \bar{x}_t=\beta\bar{x}_{t-1}+(1-\beta)x_t,
    \qquad 0<\beta<1,
    \qquad \bar{x}_0=x_0.
    \label{eq:decay-ema-update}
\end{equation}

\begin{theorem}[Exact bias--variance decomposition for decay with EMA]
\label{thm:decay-ema-exact}
Under the one-dimensional quadratic setting and inverse-time schedule of
Section~\ref{sec:lr_decay_theory}, suppose that $x_0$ is deterministic,
$0<\eta d<2$, and $0<\beta<1$. Then
\begin{align}
    \bar{x}_t
    ={}&\underbrace{\beta^t x_0
        +(1-\beta)(2-a)x_0
        \sum_{j=1}^{t}\frac{\beta^{t-j}}{D_j}}
        _{\text{mean trajectory / optimization process}}
    \nonumber\\
    &-\underbrace{\eta(1-\beta)
        \sum_{k=0}^{t-1}
        \left(\sum_{j=k+1}^{t}\frac{\beta^{t-j}}{D_j}\right)\xi_k}
        _{\text{stochastic fluctuation}}.
    \label{eq:decay-ema-closed-form}
\end{align}
Consequently,
\begin{align}
    \bar m_t:=\mathbb{E}[\bar{x}_t]
    & =\beta^t x_0+(1-\beta)(2-a)x_0
        \sum_{j=1}^{t}\frac{\beta^{t-j}}{D_j},
    \label{eq:decay-ema-mean}\\
    \bar v_t:=\operatorname{Var}(\bar{x}_t)
    & =\eta^2(1-\beta)^2\sigma^2
        \sum_{k=0}^{t-1}
        \left(\sum_{j=k+1}^{t}\frac{\beta^{t-j}}{D_j}\right)^2.
    \label{eq:decay-ema-variance}
\end{align}
Therefore,
\begin{align}
    \mathbb{E}[f(\bar{x}_t)]
    ={}&\underbrace{\frac d2\left[\beta^t x_0
        +(1-\beta)(2-a)x_0
        \sum_{j=1}^{t}\frac{\beta^{t-j}}{D_j}\right]^2}
        _{\text{optimization term, including EMA lag}}
    \nonumber\\
    &+\underbrace{\frac{d\eta^2(1-\beta)^2\sigma^2}{2}
        \sum_{k=0}^{t-1}
        \left(\sum_{j=k+1}^{t}\frac{\beta^{t-j}}{D_j}\right)^2}
        _{\text{variance contribution}}.
    \label{eq:decay-ema-loss-decomposition}
\end{align}
\end{theorem}

\begin{proof}
Unrolling the EMA recursion gives
\begin{equation}
    \bar{x}_t=\beta^t x_0
    +(1-\beta)\sum_{j=1}^{t}\beta^{t-j}x_j.
    \label{eq:decay-ema-expansion}
\end{equation}
Substitution of Eq.~\eqref{eq:decay-exact} yields
\begin{align}
    \bar{x}_t
    ={}&\beta^t x_0+(1-\beta)(2-a)x_0
    \sum_{j=1}^{t}\frac{\beta^{t-j}}{D_j}
    \nonumber\\
    &-\eta(1-\beta)
    \sum_{j=1}^{t}\frac{\beta^{t-j}}{D_j}
    \sum_{k=0}^{j-1}\xi_k.
\end{align}
Exchanging the order of summation in the stochastic term gives
\begin{equation}
    \sum_{j=1}^{t}\frac{\beta^{t-j}}{D_j}
    \sum_{k=0}^{j-1}\xi_k
    =\sum_{k=0}^{t-1}
    \left(\sum_{j=k+1}^{t}\frac{\beta^{t-j}}{D_j}\right)\xi_k,
\end{equation}
because the original index set
$\{(j,k):1\leq j\leq t,\ 0\leq k\leq j-1\}$ is equivalently
$\{(k,j):0\leq k\leq t-1,\ k+1\leq j\leq t\}$. This proves
Eq.~\eqref{eq:decay-ema-closed-form}.

For clarity, denote the deterministic part and the coefficient of each noise
variable by
\begin{align*}
A_t
&:=\beta^t x_0+(1-\beta)(2-a)x_0
  \sum_{j=1}^{t}\frac{\beta^{t-j}}{D_j},\\
c_{t,k}
&:=\eta(1-\beta)
  \sum_{j=k+1}^{t}\frac{\beta^{t-j}}{D_j}.
\end{align*}
Then Eq.~\eqref{eq:decay-ema-closed-form} becomes
\begin{equation*}
    \bar{x}_t=A_t-\sum_{k=0}^{t-1}c_{t,k}\xi_k.
\end{equation*}
Since $A_t$ and the coefficients $c_{t,k}$ are deterministic and
$\mathbb{E}[\xi_k]=0$,
\begin{align*}
\mathbb{E}[\bar{x}_t]
&=A_t-\sum_{k=0}^{t-1}c_{t,k}\mathbb{E}[\xi_k]
=A_t,
\end{align*}
which is Eq.~\eqref{eq:decay-ema-mean}. For the variance, the deterministic
term drops out, and expansion of the centered square gives
\begin{align*}
\operatorname{Var}(\bar{x}_t)
&=\mathbb{E}\!\left[
  \left(\sum_{k=0}^{t-1}c_{t,k}\xi_k\right)^2
  \right]\\
&=\sum_{k=0}^{t-1}c_{t,k}^2\mathbb{E}[\xi_k^2]
 +2\sum_{0\leq k<\ell\leq t-1}
 c_{t,k}c_{t,\ell}\mathbb{E}[\xi_k\xi_\ell]\\
&=\sigma^2\sum_{k=0}^{t-1}c_{t,k}^2\\
&=\eta^2(1-\beta)^2\sigma^2
  \sum_{k=0}^{t-1}
  \left(\sum_{j=k+1}^{t}\frac{\beta^{t-j}}{D_j}\right)^2.
\end{align*}
The third line uses independence and centering:
$\mathbb{E}[\xi_k\xi_\ell]
=\mathbb{E}[\xi_k]\mathbb{E}[\xi_\ell]=0$ for $k\neq\ell$, while
$\mathbb{E}[\xi_k^2]=\sigma^2$. This proves
Eq.~\eqref{eq:decay-ema-variance}.

Finally, because $f(x)=d x^2/2$,
\begin{align*}
\mathbb{E}[f(\bar{x}_t)]
&=\frac d2\mathbb{E}[\bar{x}_t^2]\\
&=\frac d2\left\{
  \mathbb{E}[\bar{x}_t]^2+\operatorname{Var}(\bar{x}_t)
  \right\}\\
&=\frac d2 A_t^2
  +\frac{d\sigma^2}{2}\sum_{k=0}^{t-1}c_{t,k}^2.
\end{align*}
Substituting the definitions of $A_t$ and $c_{t,k}$ gives
Eq.~\eqref{eq:decay-ema-loss-decomposition}.
\end{proof}

For every fixed $0<\beta<1$,
Eqs.~\eqref{eq:decay-ema-mean}--\eqref{eq:decay-ema-loss-decomposition}
imply
\begin{equation}
    \mathbb{E}[\bar{x}_t]=O(t^{-1}),
    \qquad \operatorname{Var}(\bar{x}_t)=O(t^{-1}),
    \qquad \mathbb{E}[f(\bar{x}_t)]=O(t^{-1}).
\end{equation}
Thus EMA preserves the contraction obtained under learning-rate decay. The
inverse-time schedule permits this closed-form calculation; the cosine-decay
experiments in Section~\ref{sec:variance_optimization_tradeoff} show a
similar empirical reduction in checkpoint variability.

\subsubsection{Smooth Nonconvex Stationarity with Learning-Rate Decay}
\label{app:decay_nonconvex}

\begin{theorem}[Nonconvex stationarity of SGD with decaying learning rates]
\label{thm:decay_nonconvex}
Let $f:\mathbb{R}^d\to\mathbb{R}$ be $L$-smooth and bounded below by
$f_\star$. Consider
\begin{equation}
    x_{t+1}=x_t-\eta_t g_t,
    \qquad g_t=\nabla f(x_t)+\xi_t,
\end{equation}
where, conditionally on the history $\mathcal{F}_t$ before sampling $g_t$,
\begin{equation}
    \mathbb{E}[\xi_t\mid\mathcal{F}_t]=0,
    \qquad
    \mathbb{E}[\|\xi_t\|^2\mid\mathcal{F}_t]\leq\sigma^2.
\end{equation}
If $0<\eta_t\leq1/L$, then for every $T\geq1$,
\begin{equation}
    \sum_{t=0}^{T-1}\eta_t
    \mathbb{E}\|\nabla f(x_t)\|^2
    \leq
    2\Delta_0+L\sigma^2\sum_{t=0}^{T-1}\eta_t^2,
    \qquad \Delta_0:=f(x_0)-f_\star.
    \label{eq:decay_nonconvex_weighted}
\end{equation}
Equivalently, if $R_T\in\{0,\ldots,T-1\}$ is sampled independently with
$\mathbb{P}(R_T=t)=\eta_t/\sum_{k=0}^{T-1}\eta_k$, then
\begin{equation}
    \mathbb{E}\|\nabla f(x_{R_T})\|^2
    \leq
    \frac{2\Delta_0+L\sigma^2\sum_{t=0}^{T-1}\eta_t^2}
         {\sum_{t=0}^{T-1}\eta_t}.
    \label{eq:decay_nonconvex_random_iterate}
\end{equation}
In particular, any schedule satisfying
$\sum_{t=0}^{\infty}\eta_t=\infty$ and
$\sum_{t=0}^{\infty}\eta_t^2<\infty$ makes the right-hand side of
Eq.~\eqref{eq:decay_nonconvex_random_iterate} converge to zero, and hence
$\mathbb{E}\|\nabla f(x_{R_T})\|^2\to0$. Since the minimum is no larger
than this weighted average, it also follows that
$\min_{0\leq t<T}\mathbb{E}\|\nabla f(x_t)\|^2\to0$.
\end{theorem}

\begin{proof}
By $L$-smoothness,
\begin{equation}
    f(x_{t+1})
    \leq f(x_t)-\eta_t\langle\nabla f(x_t),g_t\rangle
    +\frac{L\eta_t^2}{2}\|g_t\|^2.
\end{equation}
Because $x_t$ is $\mathcal{F}_t$-measurable, the conditional assumptions give
$\mathbb{E}[g_t\mid\mathcal{F}_t]=\nabla f(x_t)$. Expanding
$\|\nabla f(x_t)+\xi_t\|^2$ and using
$\mathbb{E}[\xi_t\mid\mathcal{F}_t]=0$ gives the second-moment bound
\begin{equation}
    \mathbb{E}[g_t\mid\mathcal{F}_t]=\nabla f(x_t),
    \qquad
    \mathbb{E}[\|g_t\|^2\mid\mathcal{F}_t]
    \leq\|\nabla f(x_t)\|^2+\sigma^2.
\end{equation}
Taking conditional expectation and using $\eta_t\leq1/L$ therefore gives
\begin{equation}
    \mathbb{E}[f(x_{t+1})\mid\mathcal{F}_t]
    \leq f(x_t)-\frac{\eta_t}{2}\|\nabla f(x_t)\|^2
    +\frac{L\eta_t^2\sigma^2}{2}.
\end{equation}
Taking total expectations and rearranging gives the one-step inequality
\begin{equation*}
\eta_t\mathbb{E}\|\nabla f(x_t)\|^2
\leq
2\bigl(\mathbb{E}f(x_t)-\mathbb{E}f(x_{t+1})\bigr)
+L\eta_t^2\sigma^2.
\end{equation*}
Summing it explicitly from $t=0$ to $T-1$ telescopes the function values:
\begin{align*}
\sum_{t=0}^{T-1}\eta_t
\mathbb{E}\|\nabla f(x_t)\|^2
&\leq
2\sum_{t=0}^{T-1}
\bigl(\mathbb{E}f(x_t)-\mathbb{E}f(x_{t+1})\bigr)
+L\sigma^2\sum_{t=0}^{T-1}\eta_t^2\\
&=2\bigl(f(x_0)-\mathbb{E}f(x_T)\bigr)
+L\sigma^2\sum_{t=0}^{T-1}\eta_t^2\\
&\leq2\bigl(f(x_0)-f_\star\bigr)
+L\sigma^2\sum_{t=0}^{T-1}\eta_t^2\\
&=2\Delta_0+L\sigma^2\sum_{t=0}^{T-1}\eta_t^2,
\end{align*}
which is Eq.~\eqref{eq:decay_nonconvex_weighted}.

Let $S_T:=\sum_{k=0}^{T-1}\eta_k$. Since $R_T$ is sampled independently
with $\mathbb{P}(R_T=t)=\eta_t/S_T$, the law of total expectation gives
\begin{align*}
\mathbb{E}\|\nabla f(x_{R_T})\|^2
&=\sum_{t=0}^{T-1}\mathbb{P}(R_T=t)
  \mathbb{E}\|\nabla f(x_t)\|^2\\
&=\frac{1}{S_T}\sum_{t=0}^{T-1}\eta_t
  \mathbb{E}\|\nabla f(x_t)\|^2\\
&\leq
\frac{2\Delta_0+L\sigma^2\sum_{t=0}^{T-1}\eta_t^2}{S_T},
\end{align*}
which proves Eq.~\eqref{eq:decay_nonconvex_random_iterate}. Under the stated
summability conditions, the numerator is bounded while $S_T\to\infty$, so
the bound converges to zero.
\end{proof}

\subsection{Schedule-Free}
\label{app:schedule_free}
\subsubsection{Quadratic Variance Reduction}
% MANUAL CHECK: This theorem formalizes the one-dimensional theorem/callout
% and its accompanying derivation in the source DOCX.
\begin{theorem}[Variance decay of the Schedule-Free averaged sequence]
\label{thm:schedule_free_one_dimensional}
Consider the one-dimensional quadratic objective $g(x)=\tfrac12 d x^2$, $d>0$, with stochastic gradient $g'(x)+\xi_t=dx+\xi_t$, where $\xi_t\overset{\mathrm{i.i.d.}}{\sim}\mathcal{N}(0,\sigma^2)$. Let $(y_t,z_t,x_t)$ follow the Schedule-Free recursion in Eq.~\eqref{eq:schedule_free} from deterministic initial values. If
\begin{equation}
0\leq\beta<1,\qquad a:=1-\gamma d(1-\beta)\in(0,1),
\end{equation}
then the averaged sequence $\{x_t\}$ satisfies
\begin{equation}
\mathbb{E}[x_t]=O(t^{-1}),\qquad
\operatorname{Var}(x_t)=O(t^{-1}).
\label{eq:schedule_free_one_dimensional_rates}
\end{equation}
Consequently, $\mathbb{E}[x_t^2]=O(t^{-1})$, $x_t\overset{L^2}{\to}0$, $x_t\overset{P}{\to}0$, and $\mathcal{L}(x_t)\Rightarrow\delta_0$.
\end{theorem}

\begin{proof}
Set
\begin{equation}
r=\frac{\beta}{1-\beta},\qquad a=1-\gamma d(1-\beta),\qquad
u_t=(1-a)t x_t+a z_t.
\end{equation}
Since $\gamma d\beta=(1-a)r$,
\begin{equation}
z_{t+1}=az_t-(1-a)r x_t-\gamma\xi_t,
\end{equation}
and $(t+1)x_{t+1}=t x_t+z_{t+1}$ gives
\begin{equation}
u_{t+1}=u_t-(1-a)r x_t-\gamma\xi_t.
\end{equation}
Using $x_t=(u_t-a z_t)/((1-a)t)$,
\begin{equation}
\begin{pmatrix}u_{t+1}\\z_{t+1}\end{pmatrix}
=M_t\begin{pmatrix}u_t\\z_t\end{pmatrix}-\gamma\xi_t\begin{pmatrix}1\\1\end{pmatrix},
\qquad
M_t=\begin{pmatrix}1-r/t & ar/t\\-r/t & a(1+r/t)\end{pmatrix}.
\end{equation}
For $q=r/t$ and $v=(u,z)^\top$,
\begin{align}
\|M_tv\|_2^2
&=((1-q)u+aqz)^2+(-qu+a(1+q)z)^2\\
&=(1-2q)u^2+a^2(1+2q)z^2+2q^2(u-az)^2\\
&\leq(1-2q+4q^2)u^2+a^2(1+2q+4q^2)z^2.
\end{align}
Because $q\to0$ and $a^2<1$, there is a fixed $T$ such that $\|M_tv\|_2\leq\|v\|_2$ for all $t\geq T$. Applying this to the mean recurrence bounds $\mathbb{E}[u_t]^2+\mathbb{E}[z_t]^2$ and, by Cauchy--Schwarz, gives $\mathbb{E}[x_t]=O(t^{-1})$.

For centered variables $\hat u_t$ and $\hat z_t$, define $Q_t=\mathbb{E}[\hat u_t^2+\hat z_t^2]$. Independence and zero-mean noise give
\begin{equation}
Q_{t+1}\leq Q_t+2\gamma^2\sigma^2,
\end{equation}
so $Q_t=O(t)$. Again by Cauchy--Schwarz,
\begin{equation}
\operatorname{Var}(x_t)
\leq\frac{1+a^2}{(1-a)^2t^2}Q_t=O(t^{-1}).
\end{equation}
Hence $\mathbb{E}[x_t^2]=O(t^{-1})\to0$ and $x_t\overset{L^2}{\to}0$, $x_t\overset{P}{\to}0$, and $\mathcal{L}(x_t)\Rightarrow\delta_0$.
\end{proof}

\subsubsection{Smooth Nonconvex Stationarity}
\label{app:sf_nonconvex}

The following theorem records the corresponding guarantee beyond the
quadratic setting.

\begin{theorem}[Nonconvex stationarity of Schedule-Free SGD]
\label{thm:sf_nonconvex_main}
Let $f:\mathbb{R}^d\to\mathbb{R}$ be $L$-smooth and bounded below by
$f_\star$. Suppose that, conditionally on the history $\mathcal F_t$,
the stochastic gradient evaluated at $y_t$ satisfies
\begin{equation}
\mathbb{E}[g_t\mid\mathcal F_t]=\nabla f(y_t),
\qquad
\mathbb{E}\!\left[\|g_t-\nabla f(y_t)\|^2\mid\mathcal F_t\right]
\leq\sigma^2.
\end{equation}
Starting from $x_0=z_0$, run
\begin{equation}
\begin{aligned}
y_t&=(1-\beta)z_t+\beta x_t,
& z_{t+1}&=z_t-\gamma_t g_t,
& x_{t+1}&=(1-c_{t+1})x_t+c_{t+1}z_{t+1},
\end{aligned}
\end{equation}
where $0\leq\beta<1$, $0<\gamma\leq1/L$, and, for a warmup length
$1\leq T_{\mathrm w}\leq T$ and a constant $C>0$,
\begin{equation}
\gamma_t=\gamma\min\!\left\{1,\frac{t+1}{T_{\mathrm w}}\right\},
\qquad
c_{t+1}=\min\!\left\{
\frac{(1-\beta)C\gamma_t^2}{\sum_{i=0}^{t}\gamma_i^2},1
\right\}.
\end{equation}
Define $p_{\mathrm w}:=(T_{\mathrm w}-1)/T$. Then
\begin{equation}
\begin{aligned}
\min_{0\leq t<T}\mathbb{E}\|\nabla f(y_t)\|^2
\leq{}&
\frac{2\bigl(f(x_0)-f_\star\bigr)}
{\gamma(1-\beta)(1-p_{\mathrm w})T}
+\frac{\gamma L(1-\beta)\sigma^2}{1-p_{\mathrm w}}\\
&+\frac{3\beta(2-\beta)C\sigma^2(1+\log T)}
{(1-p_{\mathrm w})T}.
\end{aligned}
\label{eq:sf_nonconvex_chae}
\end{equation}
\end{theorem}

This result controls the gradient-evaluation points $y_t$ and does not claim
last-iterate convergence of the evaluation sequence $x_T$. Its proof is given
by \citet[Theorem~4.4]{chae2026understanding}.

\section{Additional Experimental Results}
\raggedbottom

\subsection{Additional Late-Stage Comparisons}
We jointly report FID-5K and Inception Score at late-stage checkpoints, followed by the fixed-evaluation loss components for the same optimization strategies. \textsc{Schedule-Free AdamW} has lower FID-5K and higher Inception Score than \textsc{Constant LR + EMA} at every checkpoint from 320k to 400k, while the independently branched cosine-decay endpoints vary more with decay timing.

\begin{figure}[t]
\centering
\includegraphics[width=\linewidth]{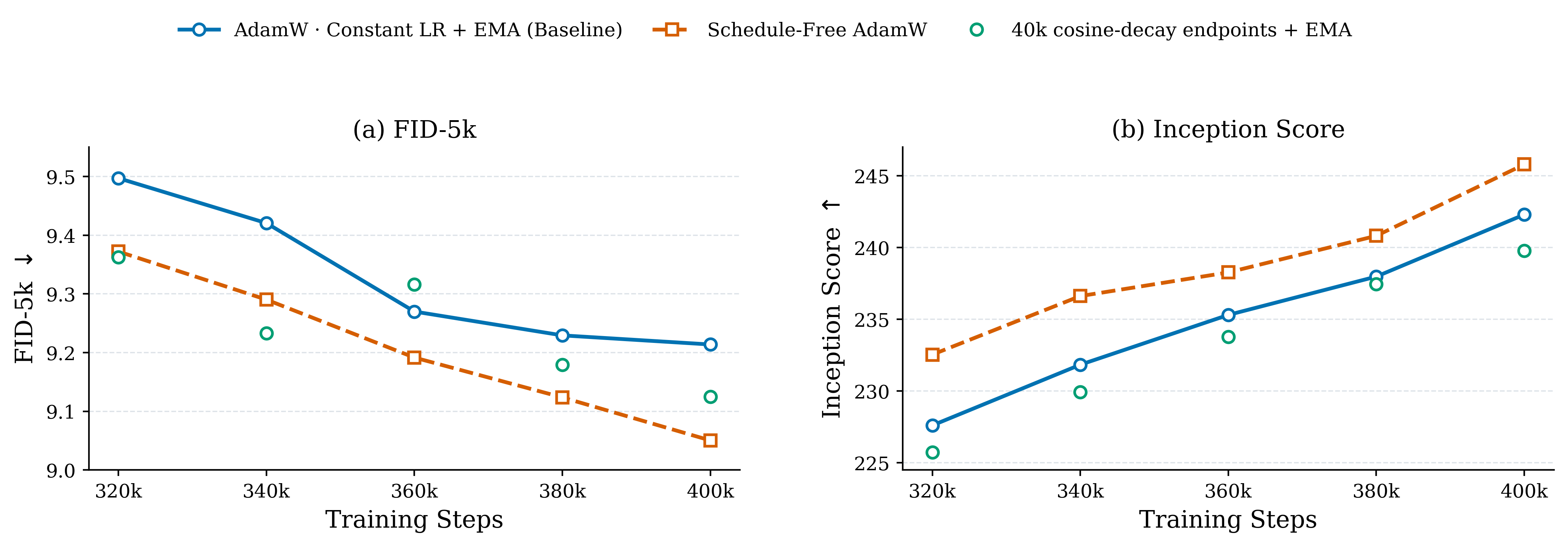}
\caption{Late-stage generation quality for \textsc{Constant LR + EMA}, \textsc{Schedule-Free AdamW}, and 40k \textsc{Decay + EMA}. (a) FID-5K, where lower is better. (b) Inception Score, where higher is better. The \textsc{Constant LR + EMA} and \textsc{Schedule-Free AdamW} curves follow continuous training trajectories. Green markers denote endpoints of independent 40k-step cosine-decay branches and should not be interpreted as a continuous trajectory. \textsc{Schedule-Free AdamW} outperforms \textsc{Constant LR + EMA} on both metrics at every checkpoint from 320k to 400k.}
\label{fig:fid_is_comparison}
\end{figure}

\begin{figure}[t]
\centering
\includegraphics[width=\linewidth]{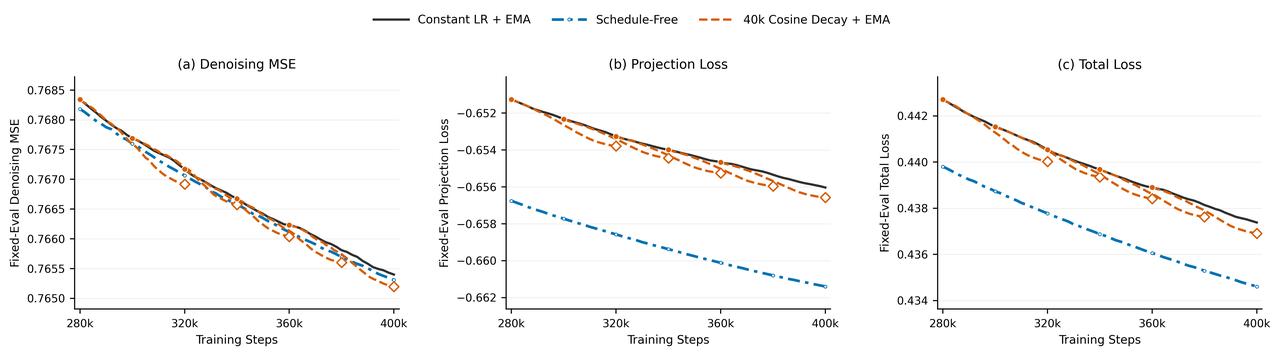}
\caption{Fixed-evaluation objectives for \textsc{Constant LR + EMA}, \textsc{Schedule-Free AdamW}, and 40k \textsc{Decay + EMA}. (a) Denoising MSE is similar across the three strategies. (b) \textsc{Schedule-Free AdamW} maintains a distinctly lower projection loss. (c) Consequently, \textsc{Schedule-Free AdamW} also achieves the lowest total fixed-evaluation loss throughout the evaluated interval. The cosine-decay values are endpoints of independent 40k-step branches; dashed segments connect them only as a visual guide.}
\label{fig:fixed_eval_three_way}
\end{figure}

\subsection{Main FID-5K Results}
\label{app:main_fid_results}
Table~\ref{tab:main_fid_results} reports the exact generation-quality
measurements underlying the comparisons in Sections~3.2--3.4, together with
the corresponding Constant LR and cosine-decay AdamW results without EMA. The
first three blocks are checkpoints from continuous training trajectories. Each
entry in either cosine-decay block is the endpoint
of a separate 40k-step branch initialized from the checkpoint shown on the
left of the arrow. The decay rows therefore compare with the
Constant LR + EMA checkpoint at the same ending step, rather than forming a
single trajectory.

\begin{table}[H]
\caption{Complete generation-quality results for the main experiment, including Constant LR AdamW without EMA and 40k-step cosine-decay AdamW evaluated with and without EMA. The first three methods follow continuous trajectories; every cosine-decay result is the endpoint of an independent branch from a different late-stage checkpoint. FID-5K entries are means over three sampling seeds. Lower FID-5K and sFID and higher Inception Score are better.}
\label{tab:main_fid_results}
\centering
\small
\begin{tabular}{llrrr}
\toprule
Method & Checkpoint / branch & FID-5K & IS & sFID\\
\midrule
\multirow{5}{*}{\textsc{Schedule-Free AdamW}}
& 320k & 9.3722 & 232.5443 & 35.1116\\
& 340k & 9.2903 & 236.6247 & 35.1530\\
& 360k & 9.1916 & 238.2731 & 35.1622\\
& 380k & 9.1237 & 240.8229 & 35.1441\\
& 400k & 9.0504 & 245.7993 & 35.1334\\
\midrule
\multirow{5}{*}{\textsc{Constant LR + EMA}}
& 320k & 9.4972 & 227.6105 & 35.0881\\
& 340k & 9.4207 & 231.8407 & 35.2050\\
& 360k & 9.2700 & 235.3101 & 35.1960\\
& 380k & 9.2294 & 237.9821 & 35.2019\\
& 400k & 9.2139 & 242.3034 & 35.1140\\
\midrule
\multirow{5}{*}{\textsc{Constant LR}}
& 320k & 13.7341 & 163.1902 & 38.1238\\
& 340k & 13.8222 & 168.1524 & 40.5599\\
& 360k & 11.5768 & 182.6766 & 35.3735\\
& 380k & 11.0776 & 190.8148 & 35.0794\\
& 400k & 11.1509 & 192.8839 & 36.3858\\
\midrule
\multirow{5}{*}{\textsc{LR Decay}}
& 280k$\to$320k & 9.7146 & 219.9526 & 35.7952\\
& 300k$\to$340k & 9.5220 & 225.3971 & 35.6992\\
& 320k$\to$360k & 9.6087 & 229.0856 & 35.7077\\
& 340k$\to$380k & 9.5660 & 231.1641 & 35.8078\\
& 360k$\to$400k & 9.4895 & 231.0582 & 36.0005\\
\midrule
\multirow{5}{*}{\textsc{Decay + EMA}}
& 280k$\to$320k & 9.3629 & 225.7234 & 35.5324\\
& 300k$\to$340k & 9.2330 & 229.9364 & 35.5310\\
& 320k$\to$360k & 9.3163 & 233.7810 & 35.4139\\
& 340k$\to$380k & 9.1795 & 237.4666 & 35.4731\\
& 360k$\to$400k & 9.1248 & 239.7905 & 35.6127\\
\bottomrule
\end{tabular}
\end{table}

Schedule-Free improves FID-5K over Constant LR + EMA at all five checkpoints,
with reductions ranging from $0.0784$ to $0.1635$, and also yields a higher
Inception Score throughout. Its sFID changes are small and not uniformly
favorable, so the improvement is not identical across every metric. Without
EMA, cosine decay substantially improves FID-5K over the corresponding
Constant LR checkpoints at all five endpoints. Cosine Decay + EMA is more
timing-sensitive relative to the Constant LR + EMA baseline: four of the five branches improve endpoint
FID-5K, whereas the branch starting at 320k is worse at 360k
($9.3163$ versus $9.2700$). The complete measurements thus support the main
text's distinction between the consistent gains of Schedule-Free and the
start-time-dependent gains of late-stage decay.

\subsection{Timestep-Wise Denoising Error}
\label{app:timestep}
To test whether the aggregate denoising improvement hides a regression at any particular diffusion noise level, we evaluate the two methods under the fixed-evaluation protocol of Appendix~\ref{app:main_config}. We reuse the same evaluation data and noise for both methods and vary only the timestep, thereby removing sampling variation from the pairwise comparison. We test
$t\in\{0.02,0.10,0.20,0.40,0.60,0.80,0.90,0.96\}$ and additionally report the result under randomly sampled timesteps. The relative improvement in the final column is
$100(\mathcal{L}_{\mathrm{Constant+EMA}}-\mathcal{L}_{\mathrm{SF}})/\mathcal{L}_{\mathrm{Constant+EMA}}$, so a positive value favors Schedule-Free.
\begin{table}[H]
\caption{Timestep-wise denoising MSE. Relative improvement is for Schedule-Free versus Constant + EMA.}
\centering
\begin{tabular}{lrrr}
\toprule
Timestep & Constant + EMA & Schedule-Free & Relative improvement\\
\midrule
random $t$ & 0.766211 & \textbf{0.765308} & 0.12\%\\
$t=0.02$ & 1.016053 & \textbf{1.013539} & 0.25\%\\
$t=0.10$ & 0.949475 & \textbf{0.947750} & 0.18\%\\
$t=0.20$ & 0.868994 & \textbf{0.867809} & 0.14\%\\
$t=0.40$ & 0.796613 & \textbf{0.795586} & 0.13\%\\
$t=0.60$ & 0.709609 & \textbf{0.708319} & 0.18\%\\
$t=0.80$ & 0.615276 & \textbf{0.614681} & 0.10\%\\
$t=0.90$ & 0.601724 & \textbf{0.601448} & 0.05\%\\
$t=0.96$ & 0.616707 & \textbf{0.616376} & 0.05\%\\
\bottomrule
\end{tabular}
\end{table}
Schedule-Free has lower denoising MSE in all eight fixed-timestep evaluations and under random timestep sampling. The relative reduction ranges from $0.05\%$ to $0.25\%$: it is largest at $t=0.02$, remains positive across the intermediate noise levels, and is still positive at $t=0.90$ and $t=0.96$. The randomly sampled evaluation improves by $0.12\%$. Thus, although the denoising differences are modest, their direction is consistent across the diffusion trajectory rather than being driven by a single timestep.

\section{Discussion}
\label{app:discussion}

Scaling has long been an important axis in diffusion-model development.
Diffusion Transformers improve systematically as model capacity and compute
increase: increasing transformer depth, width, or the number of input tokens can
lead to consistent gains in generation quality \citep{peebles2023scalable}.
More recent work has further studied scaling laws for diffusion models across
model size, data, compute, and pretraining performance
\citep{liang2024scaling,li2024scalability}. These results motivate a related
question: whether the relative effectiveness of optimization strategies may also
change across model and training scales.

This question is particularly relevant when comparing our results with the recent
controlled benchmark of diffusion optimizers in \citet{bolatov2026optimizers}.
At the end of the benchmark budget, \textsc{Schedule-Free} ranks last on the primary
metric for text8, QM9, and CelebA-64. On LM1B, Lion diverges on two of three
seeds, while \textsc{Schedule-Free} finishes last among the stable runs. This aggregate
result, however, hides several important differences. On QM9, Schedule-Free is
last in perplexity but first in Fr\'{e}chet ChemNet Distance. The experimental
regimes also differ substantially from ours. Three of the four tasks in the
benchmark use discrete diffusion---masked diffusion on text8, uniform diffusion
on QM9, and uniform diffusion trained through Gaussian duality on LM1B---while
the only Gaussian image-diffusion experiment uses a $78.7$M-parameter U-Net on
CelebA-64. The other models contain $16.6$M, $90.2$M,
and $137.1$M parameters for text8, QM9, and LM1B, respectively. In addition,
the text8 and LM1B runs use 30k training steps rather than the 1M-step budgets
used in the corresponding published studies, making the reported LM1B ordering
provisional. By contrast, our main experiments train a $675$M-parameter SiT-XL/2
Diffusion Transformer on ImageNet-1K at $256\times256$ resolution using flow
matching and REPA \citep{ma2024sit,yu2025repa}.

The optimizer dynamics also differ between the two studies. In the benchmark,
    all competing non-\textsc{Schedule-Free} optimizers are paired with a late-stage
learning-rate decay, whereas \textsc{Schedule-Free} retains a constant base learning rate
after warmup. On text8, \textsc{Schedule-Free} performs strongly during intermediate
training but is overtaken during the final decay phase. A similar late-stage
crossover appears on QM9. In our SiT+REPA setting, the ordering is reversed:
explicit decay reduces parameter variability more aggressively, whereas
\textsc{Schedule-Free} preserves more late-stage objective progress and achieves better
final generation quality. This contrast suggests that the relative ranking of
optimization strategies can depend strongly on the surrounding training regime.

At the same time, the available evidence does not identify model size or dataset
size as the cause of this reversal. The two studies differ not only in scale, but
also in diffusion formulation, architecture, training objective, learning-rate
protocol, training budget, and data domain. The current results therefore do not
support attributing the difference in optimizer ordering to any single factor.

More broadly, these observations raise the possibility that optimization
strategies themselves may exhibit scale-dependent empirical behavior. EMA
timescales, decay schedules, and Schedule-Free averaging may interact differently
with training as model size, dataset size, and optimization budget change. A
controlled study that varies model and data scale while independently retuning
each optimization strategy would be needed to determine whether optimizer
rankings follow a systematic scaling pattern. We leave this question for future
work.

\section{Additional Qualitative Results}
\label{app:qualitative}
This section complements the quantitative evaluations with samples from the
three training regimes studied in the paper. Section~\ref{app:sf_samples} shows
the main Schedule-Free pretraining result, Section~\ref{app:refl_additional_qualitative}
compares ReFL optimizers under fixed prompts, and
Section~\ref{app:dmd_sf_samples} shows samples from the generator-side
Schedule-Free DMD experiment.

\subsection{Main Pretraining: \textsc{Schedule-Free AdamW}}
\label{app:sf_samples}
Figure~\ref{fig:sf_800k_samples} provides representative class-conditional
samples from the 800k-step Schedule-Free checkpoint evaluated quantitatively in
Section~3.3. The grid spans animals, people, vehicles, objects, and natural or
built scenes rather than selecting a single semantic category. These samples
serve as a qualitative complement to the FID-50K, sFID, and Inception Score
reported in the main text.

\clearpage
\begin{figure}[t]
\centering
\includegraphics[width=\linewidth]{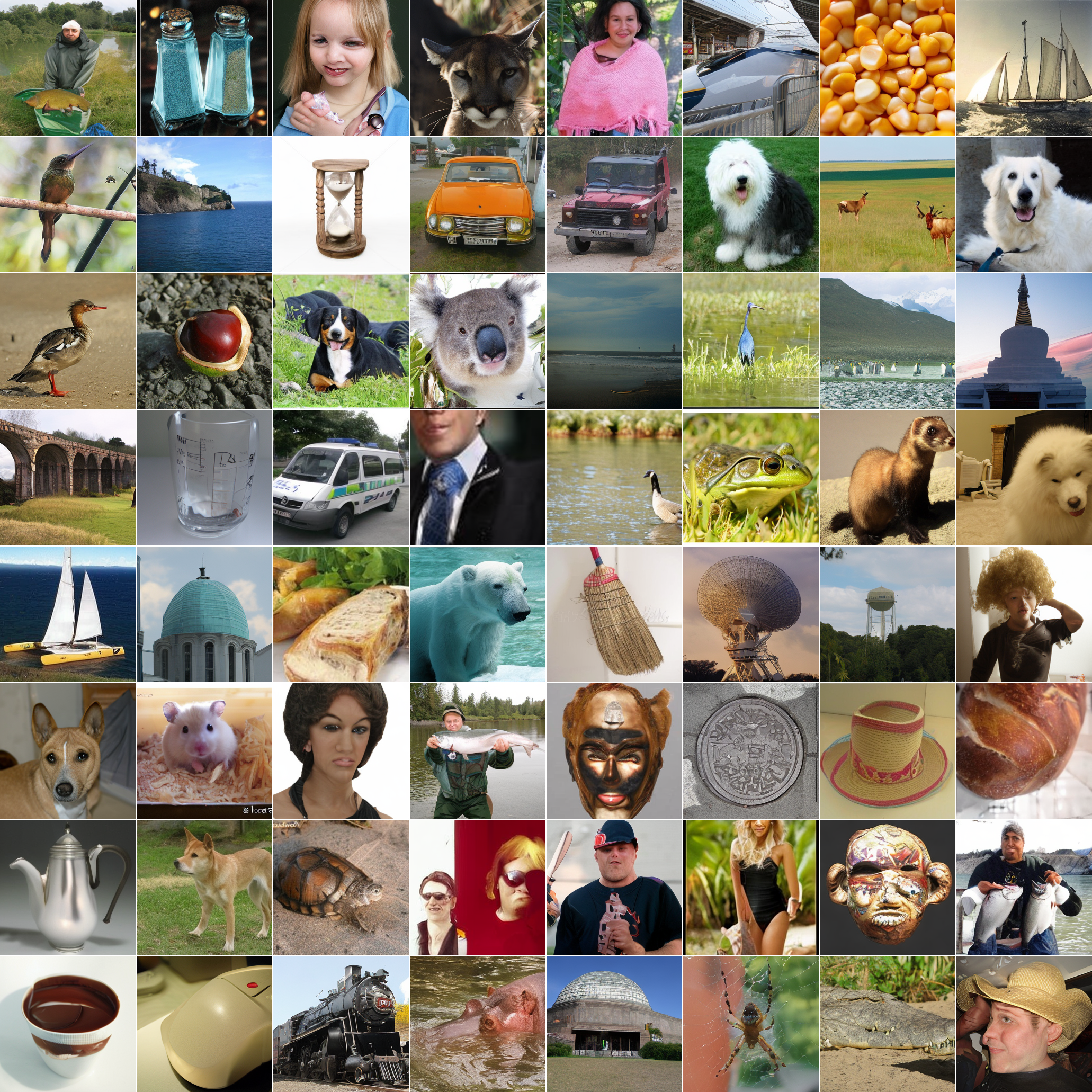}
\caption{Selected class-conditional ImageNet-1K samples from the REPA-trained SiT-XL/2 model optimized with \textsc{Schedule-Free AdamW} at $4\times$ the baseline learning rate after 800k training steps. Images are generated at $256\times256$ resolution.}
\label{fig:sf_800k_samples}
\end{figure}

\clearpage
\subsection{ReFL Qualitative Comparisons}
\label{app:refl_additional_qualitative}

We provide additional fixed-prompt comparisons for ReFL. In every figure, the rows show, from top to bottom, the Z-Image base model without reinforcement learning, AdamW at $1\times$ the reference learning rate with HPSv3, Schedule-Free AdamW at $4\times$, and Schedule-Free AdamW at $20\times$. Each column uses the prompt printed above it.

\begin{figure}[t]
\centering
\includegraphics[width=\linewidth]{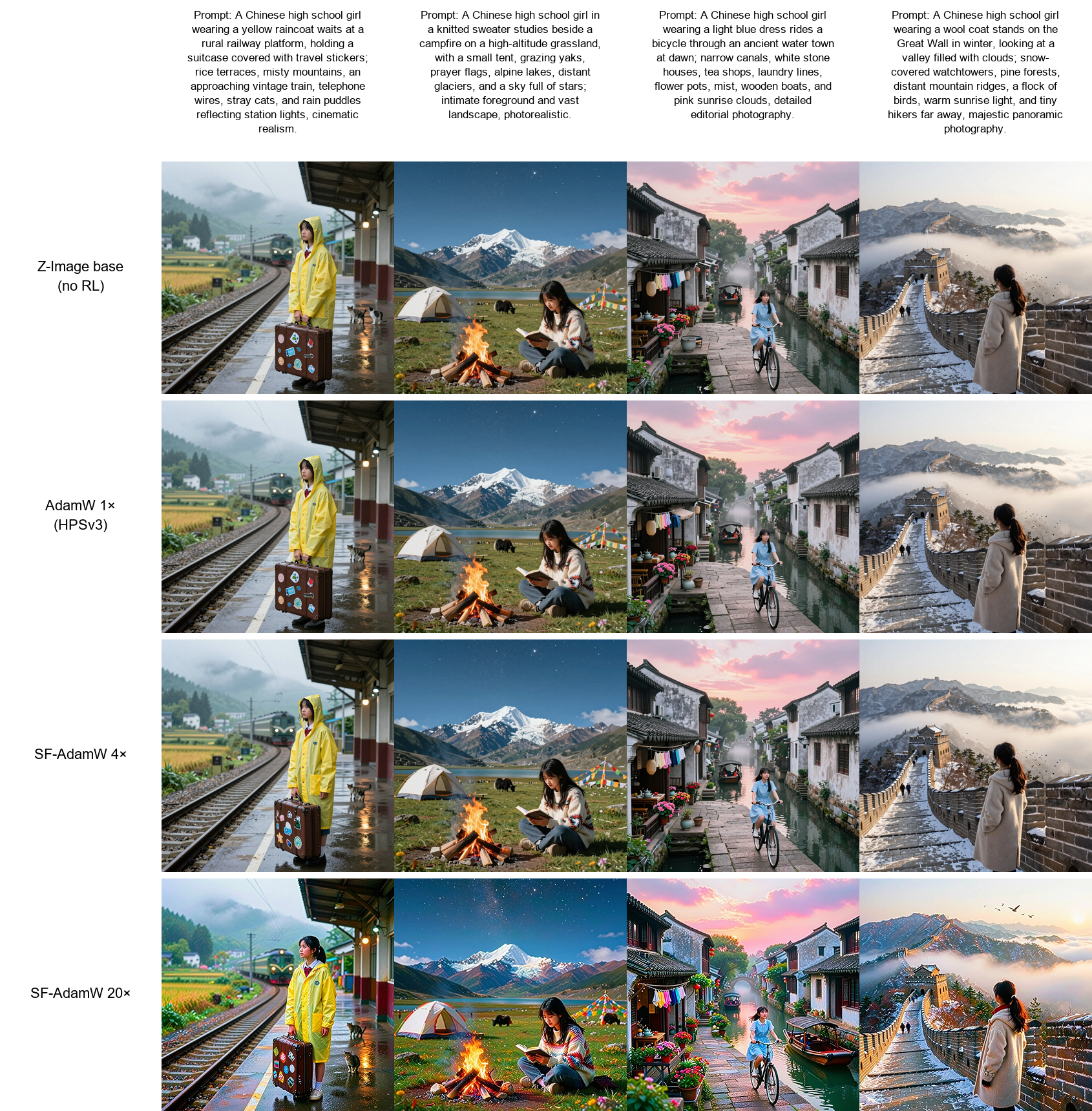}
\caption{Additional fixed-prompt ReFL qualitative comparison (prompt set 1). The rows compare the base model, AdamW, and Schedule-Free AdamW at two learning-rate scales; the complete prompts are printed above the columns.}
\label{fig:refl_additional_1}
\end{figure}

\clearpage
\begin{figure}[t]
\centering
\includegraphics[width=\linewidth]{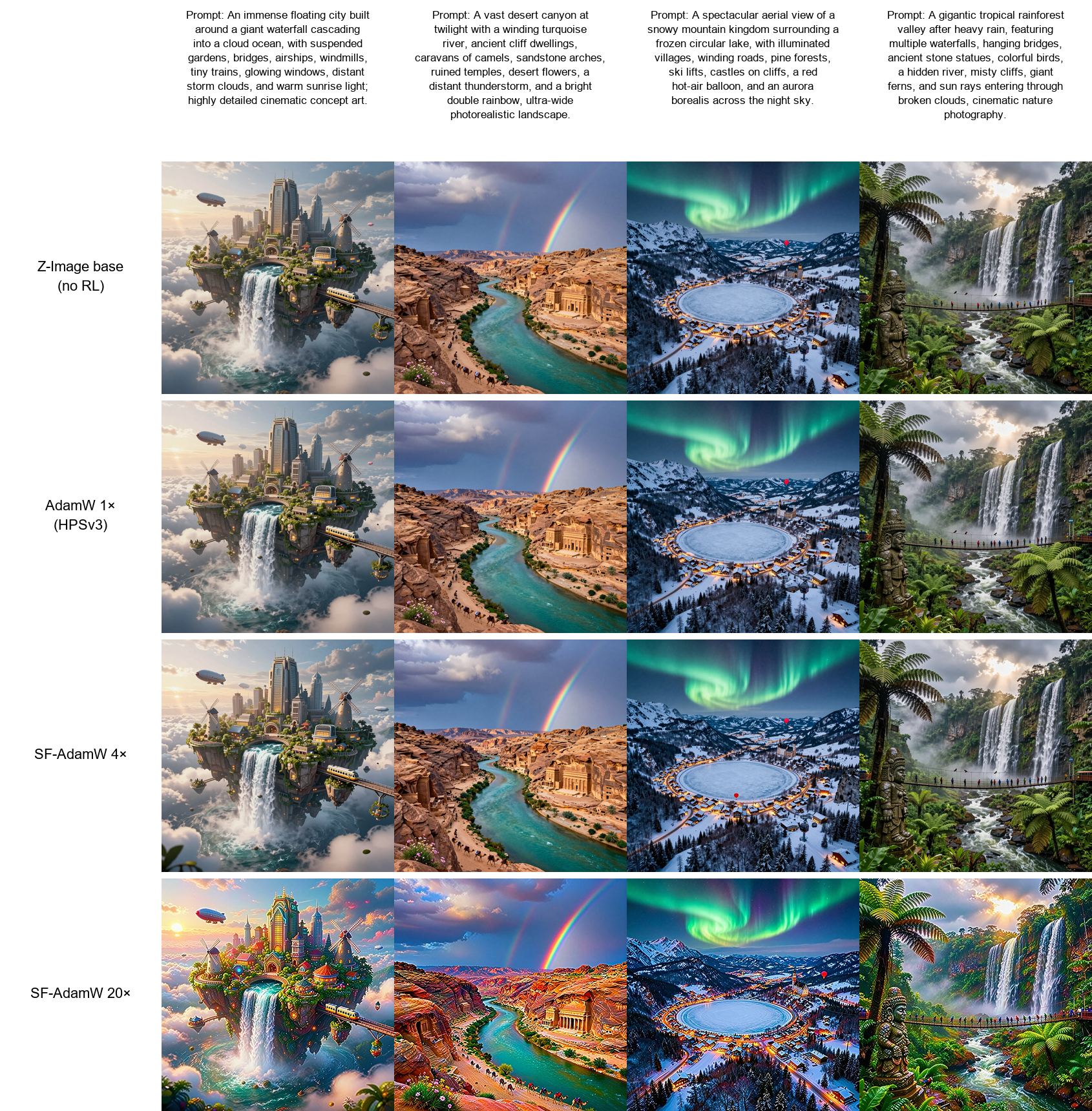}
\caption{Additional fixed-prompt ReFL qualitative comparison (prompt set 2), using the same row ordering as Figure~\ref{fig:refl_additional_1}.}
\label{fig:refl_additional_2}
\end{figure}

\clearpage
\begin{figure}[t]
\centering
\includegraphics[width=\linewidth]{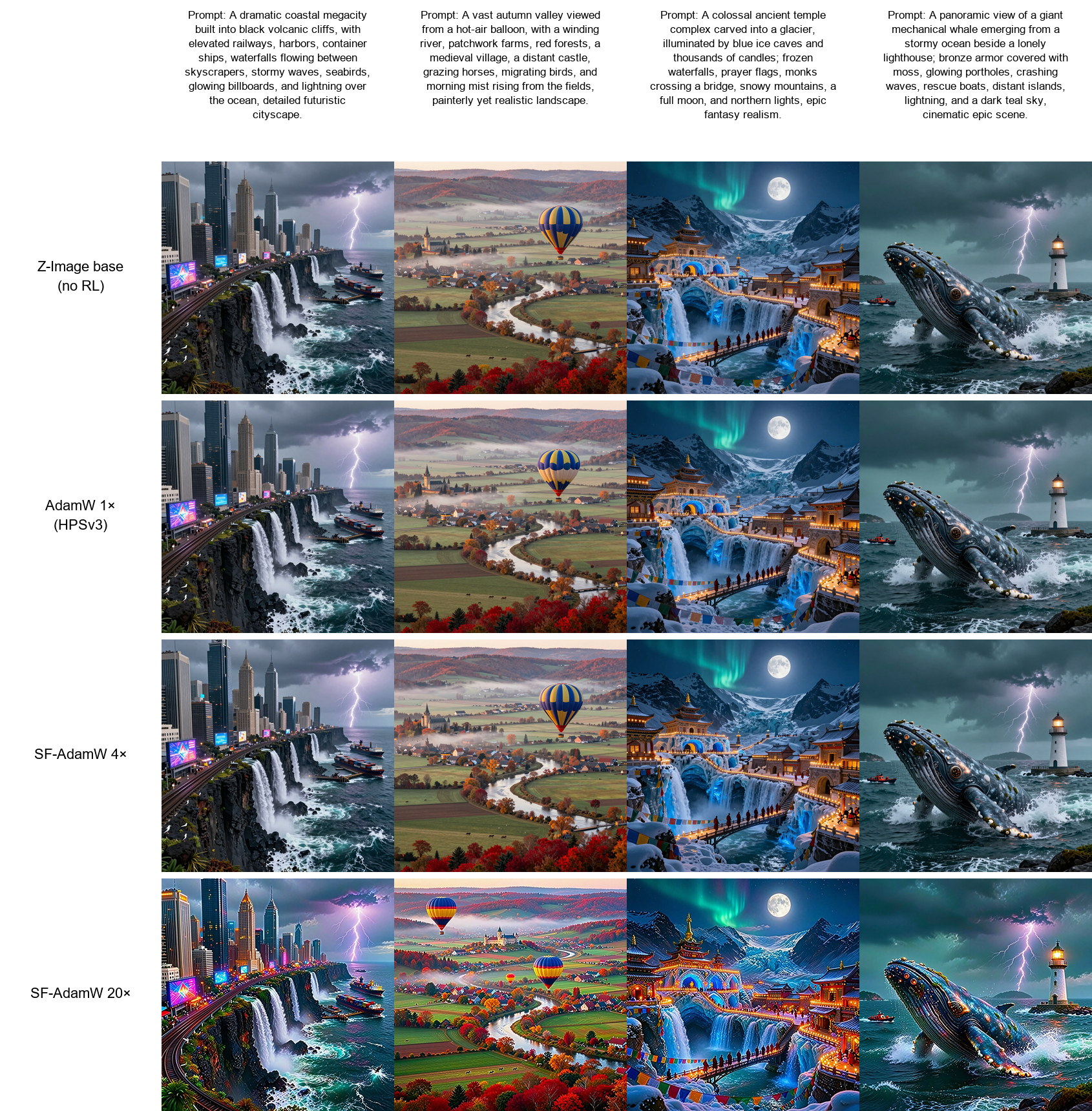}
\caption{Additional fixed-prompt ReFL qualitative comparison (prompt set 3), using the same row ordering as Figure~\ref{fig:refl_additional_1}.}
\label{fig:refl_additional_3}
\end{figure}

\clearpage
\begin{figure}[t]
\centering
\includegraphics[width=\linewidth]{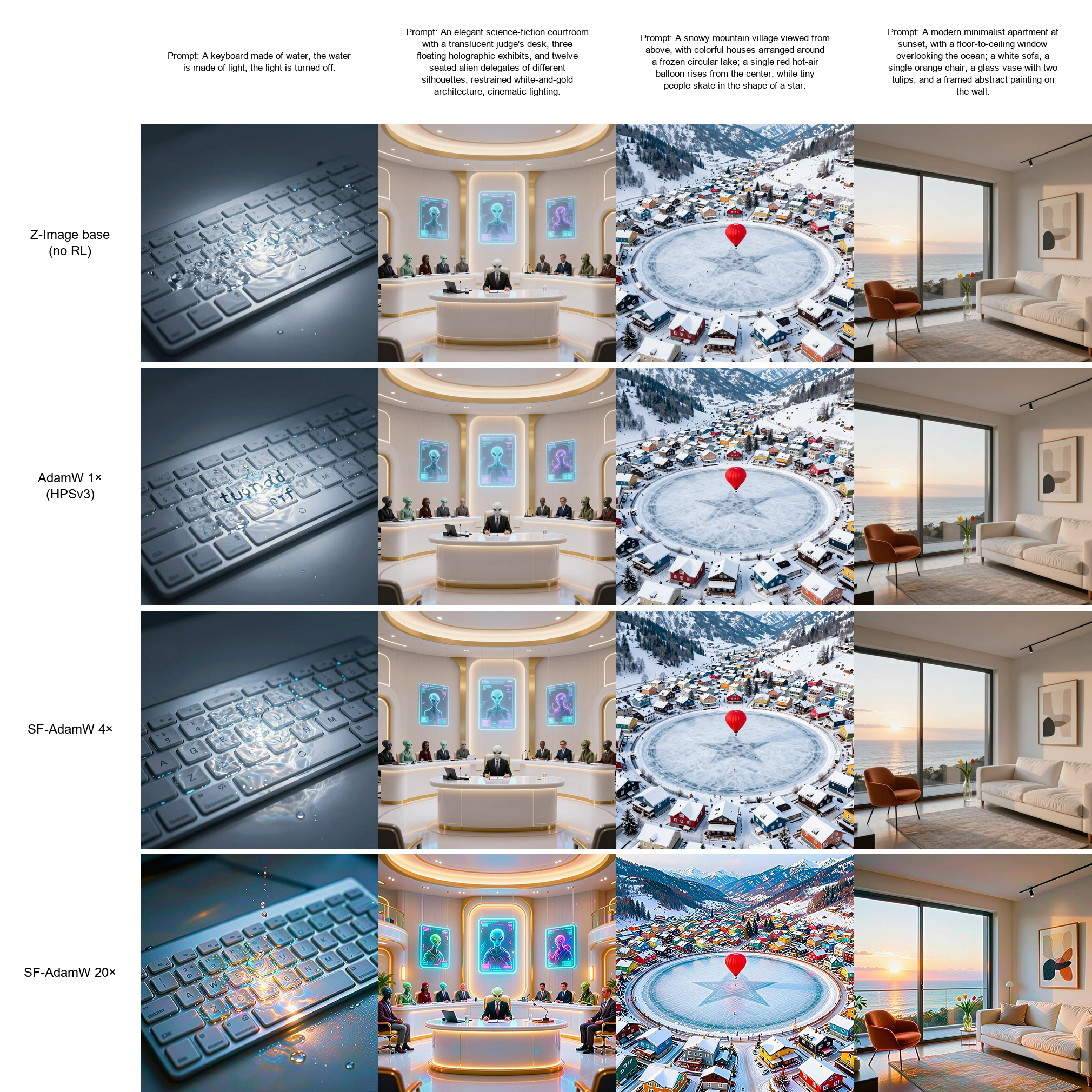}
\caption{Additional fixed-prompt ReFL qualitative comparison (prompt set 4), using the same row ordering as Figure~\ref{fig:refl_additional_1}.}
\label{fig:refl_additional_4}
\end{figure}

\clearpage
\subsection{DMD: \textsc{Schedule-Free} Generator Samples}
\label{app:dmd_sf_samples}

Figure~\ref{fig:dmd_sf_selected_grid} visualizes the best-performing variant
in the DMD optimizer ablation: Schedule-Free AdamW is applied to the generator,
while the guidance model retains the baseline AdamW optimizer. At 3k generator
steps, this configuration achieves the lowest FID-5K among the three variants
in Figure~\ref{fig:dmd_combined} and Table~\ref{tab:dmd}. The samples therefore
provide a qualitative view of the generator-side Schedule-Free improvement
discussed in Section~3.5; the complete optimizer settings are given in
Appendix~\ref{app:dmd}.

\begin{figure}[t]
\centering
\includegraphics[width=\linewidth]{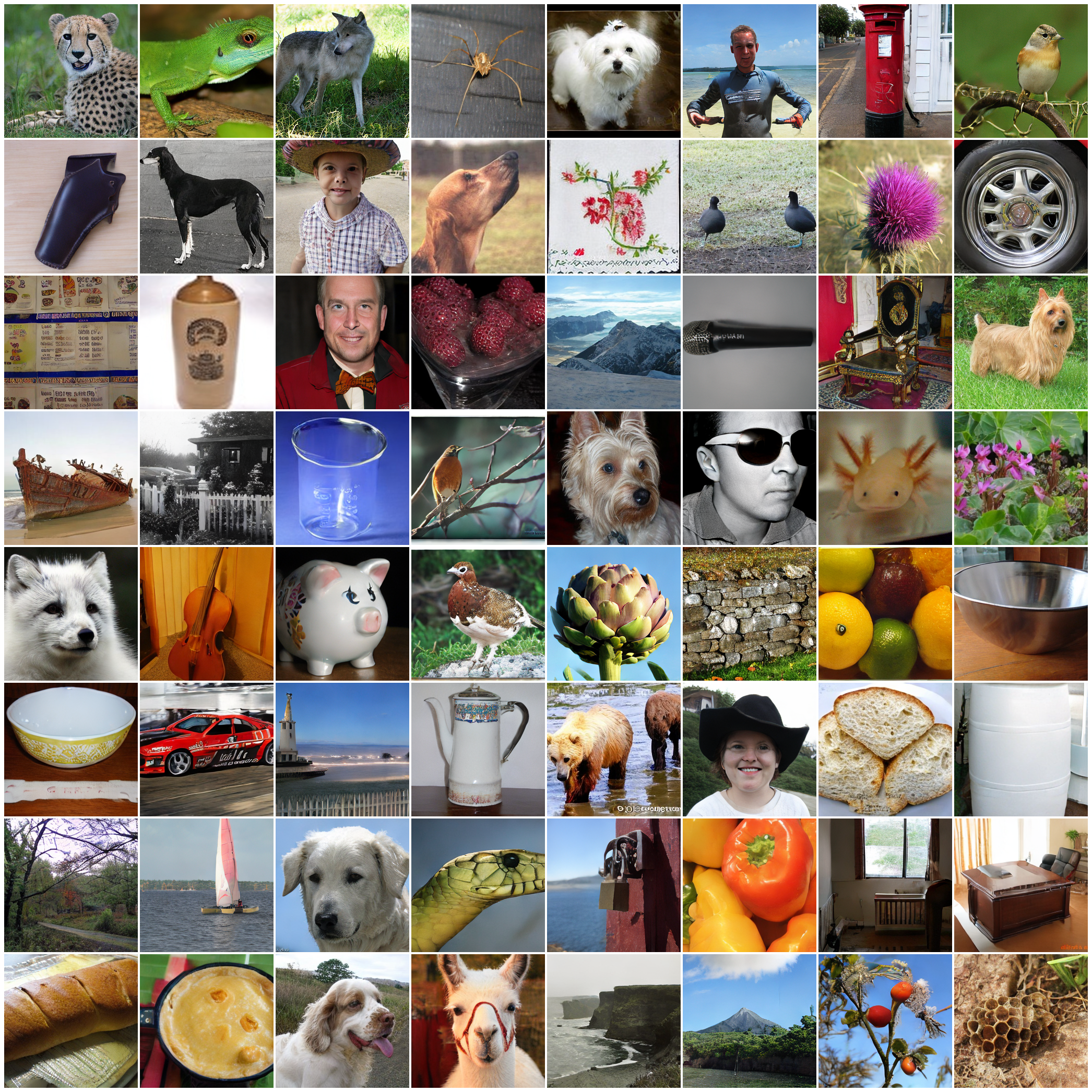}
\caption{Selected class-conditional ImageNet samples from the best-performing DMD optimizer variant after 3k generator steps. \textsc{Schedule-Free AdamW} is applied to the generator, while the guidance model retains AdamW. The figure shows an $8\times8$ grid of samples at $256\times256$ resolution.}
\label{fig:dmd_sf_selected_grid}
\end{figure}

\end{document}